\documentclass[twoside,11pt]{article}

\usepackage{blindtext}

\usepackage{graphicx} 

\usepackage[dvipsnames]{xcolor}
\colorlet{DarkGreen}{green!50!Black}
\usepackage{fullpage}
\usepackage[colorlinks]{hyperref}
\hypersetup{
    citecolor=DarkGreen,
    linkcolor=blue,
    filecolor=magenta,      
    urlcolor=cyan
    }
\usepackage{natbib}
\usepackage{amsthm}

\usepackage{amsmath,amsfonts,amssymb}
\usepackage{mathtools}

\usepackage{subcaption}

\newcommand{\sign}{\mathrm{sign}}

\newcommand{\R}{{\mathbb R}}

\newtheorem{theorem}{Theorem}
\newtheorem{proposition}[theorem]{Proposition}

\usepackage{lastpage}

\begin{document}



\title{\bf Sharpness-Aware Minimization \\
          and the Edge of Stability \\}

\author{Philip M. Long  and Peter L. Bartlett\thanks{Also affiliated with University of California, Berkeley.} \\
       Google \\
       1600 Amphitheatre Parkway \\
       Mountain View, CA 94040 \\
       $\{$plong,peterbartlett$\}$@google.com \\
       }
       

\date{}

\maketitle

\begin{abstract}
Recent experiments have shown that, often, when training a neural network with gradient
descent (GD) with a step size $\eta$, the operator norm of the Hessian of the loss
grows until it approximately reaches $2/\eta$, after which it fluctuates
around this value.  

The quantity $2/\eta$ has been called 
the ``edge of stability'' based on consideration of 
a
local quadratic approximation of the loss.  We perform a similar calculation to arrive at an ``edge of stability'' for Sharpness-Aware Minimization (SAM), a variant of GD which has been shown to
improve its generalization.  Unlike the case for GD, the resulting SAM-edge
depends on the norm of the gradient.  Using three deep learning training tasks,
we see empirically that SAM operates on the edge of stability identified by
this analysis.
\end{abstract}


\section{Introduction}

{\em Sharpness-aware Minimization} (SAM) 
\citep{foret2020sharpness} is a new gradient-based neural network
training algorithm that advanced the
state-of-the-art test accuracy on a number of prominent 
benchmark datasets.  As its name suggests, 
it explicitly seeks
to find a solution that not only fits the training data, but that
avoids ``sharp'' minima, for which nearby parameter vectors 
perform poorly.  
SAM is an incremental algorithm that updates its parameters using a gradient
computed at a neighbor of the current solution.  The neighbor is
the point in parameter space found by taking a step of length $\rho$ ``uphill'' in the gradient direction.  The practical success of SAM has motivated
theoretical research \citep{bartlett2022dynamics,wen2023how,andriushchenko2023sharpness},
including results highlighting senses in which SAM's update may be
viewed, under certain conditions, as including a component that
performs gradient descent on the operator norm of the Hessian
\citep{bartlett2022dynamics,wen2023how}.

Meanwhile, 
\citet{cohen2021gradient},
building on the work of \citet{jastrzebski2019break} and others,
exposed a striking phenomenon regarding neural network training
with the original gradient descent (GD) method: for many initialization schemes
and learning rates $\eta$, the operator norm of the Hessian
eventually settles in the neighborhood of $2/\eta$.  This has
been termed the ``edge of stability'', in part because a
convex quadratic trained by gradient descent with a learning rate
$\eta$ will only converge if the operator norm of its Hessian
(which is the same everywhere) is less than $2/\eta$.  
This phenomenon also inspired 
substantial theoretical research
\citep[see][]{arora2022understanding,DNL22,ma2022beyond,zhu2022understanding,ahn2022understanding,chen2022gradient}.
One result identified conditions under which,
when training approaches the edge of stability, the dynamics
includes a self-stabilization mechanism that tends to drive
the operator norm of the Hessian back down \citep{DNL22}.

In this paper, we investigate whether SAM operates
at the edge of stability.  First, we perform a derivation, analogous
to the one that identifies $2/\eta$ as the edge of stability for
GD, that yields a formula for the operator norm of the Hessian
that may be viewed as the edge of stability for SAM.  As expected,
SAM's edge of stability depends on the radius $\rho$ of its neighborhood.
It also depends
on the norm of the gradient of
the training error at the current solution, unlike the case of GD.
As the norm of the gradient gets smaller, the edge
of stability for SAM also gets smaller.

Next, we evaluate experimentally whether SAM operates at the edge of
stability identified by our analysis.  Our first experiments are with
fully connected networks on MNIST.  Here, it is feasible to experiment
with a version of SAM that uses a batch gradient, albeit computed at
the neighbor uphill of the current iterate at a distance $\rho$.  
For many combinations of the step size $\eta$ and the radius $\rho$,
the operator norm of the Hessian at SAM's iterates closely matches
the value arising from our analysis.  Next, we experiment with
a convolutional neural network training on 1000 examples from CIFAR10.
Here again, we see SAM operating on the edge of stability.
Finally, we experiment with a standard
Transformer architecture training a language model on
\verb|tiny_shakespeare| using the more practical version of
SAM that uses stochastic gradients.  Here, we also see substantial
agreement with our theoretical analysis.

In our experiments with SAM, its edge of stability is often
{\em much} smaller than $2/\eta$, even early in training.  
Rather than first driving the training error to a very small value,
and then drifting along a manifold of near-optimal solutions to
wider minima, SAM's process drives solutions toward smooth regions
of parameter space early in training, while the loss is still large.

The derivation of SAM's edge of stability is in Section~\ref{s:derivation}.
The experiments are described in detail in Section~\ref{s:methods}.  
The results are in Section~\ref{s:results}.  
Section~\ref{s:related}
includes further description of related work.
We conclude in Section~\ref{s:conclusion}.


\section{Derivation}
\label{s:derivation}

The {\em Sharpness-Aware Minimization} algorithm 
is defined by the update
   \begin{equation}\label{e:sam}
      w_{t+1} = w_t-\eta\nabla \ell\left(w_t + \rho\frac{\nabla
        \ell(w_t)}{\|\nabla \ell(w_t)\|}\right).
    \end{equation}
This is like gradient descent, except using a
gradient evaluated at
$w_t + \rho\frac{\nabla
        \ell(w_t)}{\|\nabla \ell(w_t)\|}$
instead of $w_t$.

In this section, we calculate an ``edge of stability'' for
SAM analogous to the $2/\eta$ value for GD.

Before analyzing SAM, however, let us review the standard analysis that
identifies the edge of stability for GD, assuming for simplicity that the
quadratic approximation around an iterate is exact.
\begin{proposition}
\label{p:gd_edge}
For $w_t \in \R^d$, $\eta > 0$, if
\begin{itemize}
    \item $g = \nabla \ell(w_t) \neq 0$, $H = \nabla^2 \ell(w_t)$,
         $w_{t+1} = w_t - \eta g$, and
    \item for all $w \in \R^d$, $\ell(w) = \ell(w_t) + g^T (w - w_t) + \frac{(w - w_t)^{\top} H (w - w_t)}{2}$,
\end{itemize}
then
\begin{itemize}
\item if $|| H ||_{op} < \frac{2}{\eta}$, then $\ell(w_{t+1}) < \ell(w_t)$, and
\item this condition on $|| H ||_{op}$ is the weakest possible of its type: if
   \begin{itemize}
       \item $g$ is aligned with a principal eigenvector of $H$ whose eigenvalue is non-negative, then 
       \item $\sign(\ell(w_{t+1}) - \ell(w_t))
                   = \sign\left(|| H ||_{op} - \frac{2}{\eta}\right)$.
    \end{itemize}
\end{itemize}
\end{proposition}
\begin{proof}
Substituting $w_{t+1} - w_t$ into the formula for $\ell$, we have
\begin{align*}
\ell(w_{t+1}) 
  & =  \ell(w_t) - \eta g^{\top} g 
            + \frac{\eta^2 g^{\top} H g}{2} \\
  & \leq \ell(w_t) - \eta g^{\top} g + \frac{\eta^2 g^{\top} || H ||_{op} g}{2} \\
  & = \ell(w_t) - \eta \left(1 -  \frac{\eta || H ||_{op} }{2} \right)|| g ||^2.
\end{align*}
If $|| H ||_{op} < \frac{2}{\eta}$, since $g \neq 0$, this implies $\ell(w_{t+1}) < \ell(w_t)$.

When $g$ is aligned with a principal eigenvector of $H$
whose eigenvalue is non-negative,
we have
$H g = || H ||_{op} g$, which implies, as above, that
\[
\ell(w_{t+1}) = \ell(w_t) - \eta \left(1 -  \frac{\eta || H ||_{op} }{2} \right)|| g ||^2,
\]
which, again since $g \neq 0$, implies 
$\sign(\ell(w_{t+1}) - \ell(w_t)) 
                   = \sign\left(|| H ||_{op} - \frac{2}{\eta}\right)$.
\end{proof}
We can think of Proposition~\ref{p:gd_edge}
as formalizing the statement that $2/\eta$
is the edge of stability for GD: if $|| H ||_{op} < 2/\eta$,
GD is guaranteed to make progress, and no
larger bound suffices.

Even in the convex quadratic case, the dynamics of SAM are much more complex
than GD \citep[see][]{bartlett2022dynamics}.  However, if we bound $|| H ||_{op}$ in terms of $|| g ||$ as well
as $\eta$ and $\rho$, an analogous result
holds.
\begin{proposition}
\label{p:sam_edge}
For $w_t \in \R^d$, $\eta > 0$, $\rho > 0$, if
\begin{itemize}
    \item $g = \nabla \ell(w_t) \neq 0$, $H = \nabla^2 \ell(w_t) \succeq 0$,
         $w_{t+1} = w_t-\eta\nabla \ell\left(w_t + \rho\frac{\nabla
        \ell(w_t)}{\|\nabla \ell(w_t)\|}\right)$, and
    \item for all $w \in \R^d$, $\ell(w) = \ell(w_t) + g^T (w - w_t) + \frac{(w - w_t)^{\top} H (w - w_t)}{2}$,
\end{itemize}
then
\begin{itemize}
\item if $|| H ||_{op} < \frac{  || g ||}{2 \rho} 
  \left(\sqrt{1 + \frac{8 \rho }{ \eta || g ||}}-1\right)$, then $\ell(w_{t+1}) < \ell(w_t)$, and
\item this condition on $|| H ||_{op}$ is the weakest possible of its type:
if
   \begin{itemize}
       \item $g$ is aligned with a principal eigenvector of $H$, then 
       \item $\sign(\ell(w_{t+1}) - \ell(w_t))
                   = \sign\left(|| H ||_{op} - \frac{  || g ||}{2 \rho} 
  \left(\sqrt{1 + \frac{8 \rho }{ \eta || g ||}}-1\right)\right)$.
    \end{itemize}
\end{itemize}
\end{proposition}

Proposition~\ref{p:sam_edge} is an immediate consequence
of the following stronger, but somewhat more technical, proposition.
\begin{proposition}
\label{p:sam_edge_strong}
For $w_t \in \R^d$, $\eta > 0$, $\rho > 0$, if
\begin{itemize}
    \item $g = \nabla \ell(w_t) \neq 0$ and $H = \nabla^2 \ell(w_t)$
    has eigenvalues $\lambda_1,...,\lambda_d$
    and unit-length eigenvectors $v_1,...,v_d$,
    \item
         $w_{t+1} = w_t-\eta\nabla \ell\left(w_t + \rho\frac{\nabla
        \ell(w_t)}{\|\nabla \ell(w_t)\|}\right)$, 
    \item for all $w \in \R^d$, $\ell(w) = \ell(w_t) + g^T (w - w_t) + \frac{(w - w_t)^{\top} H (w - w_t)}{2}$
\end{itemize}
then
\begin{itemize}
\item if, for all $i$, 
\[
-\frac{|| g ||}{\rho}
 \leq
 \lambda_i
 \leq
\frac{  || g ||}{2 \rho} 
  \left(\sqrt{1 + \frac{8 \rho }{ \eta || g ||}}-1\right),
\]
and there is an $i$ such that
\[
g \cdot v_i \neq 0 \mbox{ and }
-\frac{|| g ||}{\rho}
 <
 \lambda_i
 <
\frac{  || g ||}{2 \rho} 
  \left(\sqrt{1 + \frac{8 \rho }{ \eta || g ||}}-1\right),
\]
then $\ell(w_{t+1}) < \ell(w_t)$,
and
\item if
   \begin{itemize}
       \item $g$ is aligned with a principal eigenvector of $H$ whose eigenvalue is non-negative, then 
       \item $\sign(\ell(w_{t+1}) - \ell(w_t))
                   = \sign\left(|| H ||_{op} - \frac{  || g ||}{2 \rho} 
  \left(\sqrt{1 + \frac{8 \rho }{ \eta || g ||}}-1\right)\right)$.
    \end{itemize}
\end{itemize}
\end{proposition}
\begin{proof}
Substituting $w_{t+1} - w_t$ into the formula for $\ell$, in part since
$H$ is symmetric, we have
\begin{align*}
\ell(w_{t+1})
 & = \ell(w_t) - \eta g^{\top} \left( g + \rho H \frac{g}{|| g ||} \right)
            + \frac{\eta^2 \left( g + \rho H \frac{g}{|| g ||} \right)^{\top} H \left( g + \rho H \frac{g}{|| g ||} \right)}{2} \\
 & = \ell(w_t)
     - \eta g^{\top} 
           \left(
             I +  \frac{\rho H}{|| g ||} 
            - \eta 
              \left(
              \frac{\left( I +  \frac{\rho H}{|| g ||} \right)^2 H}{2}
              \right)
            \right)
              g.
\end{align*}
Using the fact that, since $H$ is symmetric, any matrix polynomial of $H$
has the same eigenvectors as $H$,
we have
\begin{align}
\nonumber
\ell(w_{t+1})
 & = \ell(w_t)
    - \eta \sum_{i=1}^n 
       (v_i \cdot g)^2
       \left(
             1 +  \frac{\rho \lambda_i}{|| g ||} 
            - \eta 
              \left(
              \frac{\left( 1 +  \frac{\rho \lambda_i}{|| g ||} \right)^2 \lambda_i}{2}
              \right)
       \right) \\    
\label{e:by_directions}
 & = \ell(w_t)
    - \eta \sum_{i=1}^n 
       (v_i \cdot g)^2
       \left(
             1 +  \frac{\rho \lambda_i}{|| g ||}  \right)
        \left( 1
            - 
              \frac{\eta \left( 1 +  \frac{\rho \lambda_i}{|| g ||}
                 \right) \lambda_i}{2}
              \right).
\end{align}
Recalling that each $\lambda_i \geq -\frac{|| g ||}{\rho}$,
let us focus on the last factor of one term in
the sum of \eqref{e:by_directions}
for which $\lambda_i > -\frac{|| g ||}{\rho}$ and
$(v_i \cdot g)^2 \neq 0$.  We have
\begin{align*}
&  1
            - 
              \frac{\eta \left( 1 +  \frac{\rho \lambda_i}{|| g ||}
                 \right) \lambda_i}{2} \geq 0 \\
& \Leftrightarrow 
  \eta \rho \lambda_i^2
     + \eta \lambda_i || g || - 2 || g || \leq 0.
\end{align*}
The convex quadratic on the LHS has two solutions,
one that is negative, and one that is positive:
\begin{align*}
& \frac{\pm\sqrt{\eta^2 || g ||^2 + 8 \eta \rho || g ||}-\eta || g ||}{2 \eta \rho} \\
& = \frac{  || g ||}{2 \rho} 
  \left(\pm \sqrt{1 + \frac{8 \rho }{ \eta || g ||}}-1\right).
     \\
\end{align*}
Thus, given that $\lambda_i > -\frac{|| g ||}{\rho}$,
the $i$th term of the sum in \eqref{e:by_directions} is positive
iff 
\begin{equation}
    \label{eqn:lambdabound}
-\frac{  || g ||}{2 \rho} 
  \left(\sqrt{1 + \frac{8 \rho }{ \eta || g ||}}+1\right)
  < 
\lambda_i < \frac{  || g ||}{2 \rho} 
  \left(\sqrt{1 + \frac{8 \rho }{ \eta || g ||}}-1\right),
\end{equation}
for which
\[
-\frac{  || g ||}{\rho} 
  < 
\lambda_i < \frac{  || g ||}{2 \rho} 
  \left(\sqrt{1 + \frac{8 \rho }{ \eta || g ||}}-1\right),
\]
suffices.  Thus each  term in the sum of
\eqref{e:by_directions} is non-negative, and
at least one is positive, so $\ell(w_{t+1}) < \ell(w_t)$.

If $g$ is aligned with a principal eigenvector of $H$ whose
eigenvalue is non-negative, assuming
wlog that this principal eigenvector is $v_1$, 
we have $(v_1 \cdot g)^2 > 0$, and
$(v_i \cdot g)^2 = 0$ for all
$i \neq 1$.  In this case, all of the terms in the
sum in \eqref{e:by_directions} are zero except the first, thus
\begin{align*}
\sign(\ell(w_{t+1}) - \ell(w_t))
 & = -\sign\left( 1
            - 
              \frac{\eta \left( 1 +  \frac{\rho \lambda_1}{|| g ||}
                 \right) \lambda_1}{2}
              \right) \\
 & = \sign\left(\lambda_1 - \frac{  || g ||}{2 \rho} 
  \left(\sqrt{1 + \frac{8 \rho }{ \eta || g ||}}-1\right)\right) \\ 
 & = \sign\left(|| H ||_{op} - \frac{  || g ||}{2 \rho} 
  \left(\sqrt{1 + \frac{8 \rho }{ \eta || g ||}}-1\right)\right), \\
\end{align*}
where we have used the equivalent bounds on $\lambda_1$ given by~\eqref{eqn:lambdabound}.
\end{proof}

We refer to the
threshold
$\frac{  || g ||}{2 \rho} 
  \left(\sqrt{1 + \frac{8 \rho }{ \eta || g ||}}-1\right)$
  identified in Proposition~\ref{p:sam_edge}
  as {\em SAM's edge of stability},
  or the SAM-edge for short.

The ratio
$\frac{|| H ||_{op}}{2/\eta}$ between
the edge of stability for SAM, and the edge for GD, is
\[
\frac{|| H ||_{op}}{2/\eta}
 = \frac{\eta \|g\|}{4 \rho }\left(\sqrt{1 + \frac{8\rho}{\eta\|g\|}} - 1\right).
\]
This ratio depends on $\eta$, $\rho$ and $|| g ||$
through $\eta\|g\|/(2\rho)$; let us refer to this intermediate
quantity as $\alpha$.
Figure~\ref{fig:SAMedgeplot} shows the function
\[
\alpha\mapsto \frac{\alpha}{2}\left(\sqrt{1+\frac{4}{\alpha}}-1\right),
\]
that, at SAM's edge of stability, gives $\|H\|_{op}/(2/\eta)$ as a function of $\alpha=\eta\|g\|/(2\rho)$.
Notice that as $\alpha\to\infty$, this function approaches $1$, and
it approaches zero like $\sqrt{\alpha}$.

\begin{figure}[ht]
    \centering
    \hfill
    \includegraphics[width=0.4\textwidth]{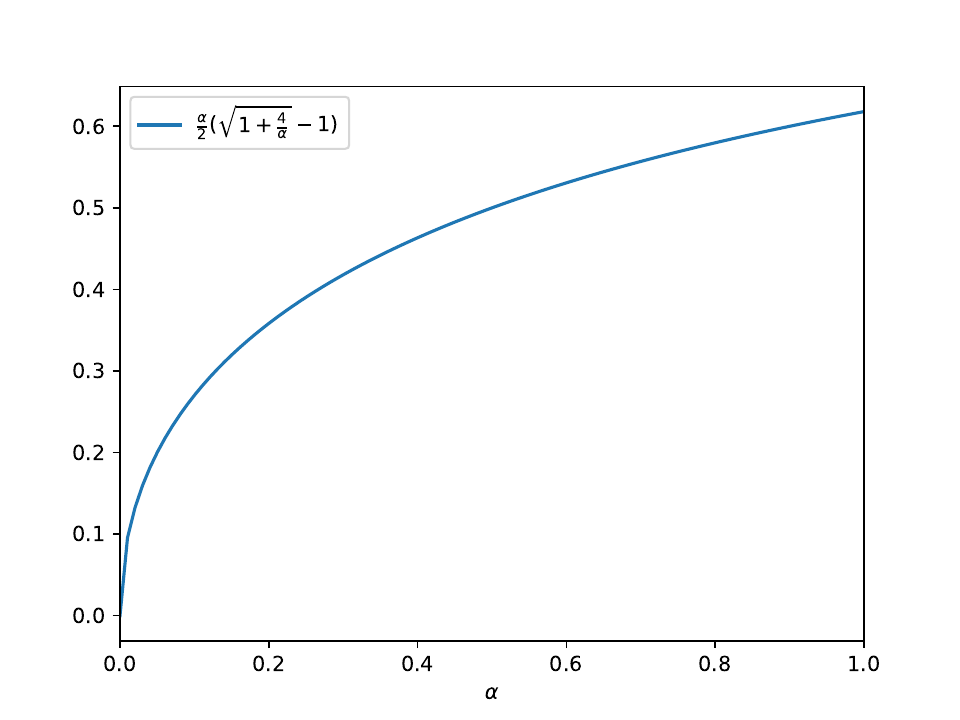}\hfill
    \includegraphics[width=0.4\textwidth]{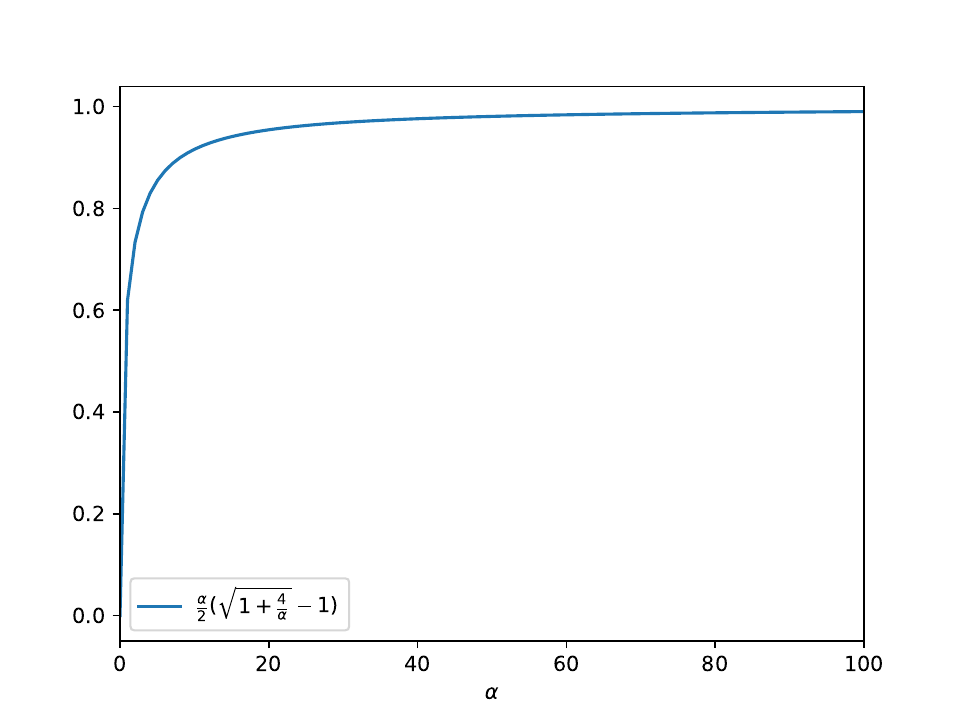}
    \hfill
    \caption{\label{fig:SAMedgeplot}The ratio of 
    SAM's edge of stability to $2/\eta$, the
    edge of stability for GD,
    as a function of $\alpha=\eta\|g\|/(2\rho)$.}
\end{figure}

Proposition~\ref{p:sam_edge_strong} focuses on the case where the largest
eigenvalue is positive.  This is motivated in part by the work of
\citet{ghorbani2019investigation}, who found that, often,
after a small amount of training of a neural network, any negative eigenvalues in the Hessian
are very small. 

\section{Methods}
\label{s:methods}

We performed experiments in three settings.  In each setting, we trained for a variety of
combinations of hyperparameters, and tracked various quantities, including the operator norm of the
Hessian, and the SAM edge.  
\href{https://github.com/google-deepmind/sam_edge}{Code}
is available
\citep{LB24code}.

\subsection{Settings}

First, we trained a depth-four fully connected network,
with 1000 nodes in each hidden layer, on MNIST using the quadratic loss
with batch gradient descent.  We trained for eight hours of wallclock time
on a V100 GPU.  The weights were initialized using Glorot normal initialization.
Prior to
training, the data was centered.

Next, we trained a CNN on CIFAR10 using the
quadratic loss.  To make batch gradients feasible, we only trained on
the first 1000 examples.
The CNN architecture was standard:
there were two blocks comprised of
a convolutional layer with a
ReLU nonlinearity followed by layer normalization,
then $2 \times 2$ max pooling with a $2 \times 2$
stride.  In the first block the convolutional layer
had 16 channels, and in the second block, it had 32 channels.  Training was performed for 12 hours on a V100
GPU.  
Here again, the weights were initialized using Glorot normal initialization, and 
data was centered before training.

For the final setting, we modified the sample implementation of 
Transformers distributed with the Haiku package  \citep[see][]{haiku2023transformer},
training
an autoregressive character language model using the \verb+tiny_shakespeare+ dataset,
using minibatches of size 128.  
The operator norm of the Hessian, and its principal
directions, were also estimated using minibatches.
The architecture
was as in the Haiku distribution, with
6 layers, 8 heads, a key size of 32, ``model size'' of 128,
and sequence length of 64.  
Because it introduces noise, Dropout was removed.
The last 10000 lines
of \verb+tiny_shakespeare+ were set aside as a test set,
and the remaining data was used for training.

\subsection{Hyperparameters}

We trained once for each combination of the following hyperparameters:
\begin{itemize}
\item For MNIST,
   \begin{itemize}
     \item learning rates $\eta$: 0.03, 0.1, 0.3,
     \item SAM offsets $\rho$ (see \eqref{e:sam}): 0.0, 0.1, 0.3, 1.0.
   \end{itemize}
\item For CIFAR10,
   \begin{itemize}
     \item learning rates: 0.0003, 0.001, 0.003, 0.01,
     \item $\rho$ values: 0.0, 0.1, 0.3, 1.0
   \end{itemize}
\item For \verb+tiny_shakespeare+,
   \begin{itemize}
       \item learning rates: 0.01, 0.02, 0.05, 0.1, 0.2, 0.5
       \item $\rho$ values: 0.0, 0.1, 0.3, 1.0.
   \end{itemize}
\end{itemize}

Results were discarded whenever training diverged.

\subsection{Implementation}

We coded our experiments using Jax \citep{jax2018github}, along with
Flax \citep{flax2020github} (for the image classification experiments), and
Haiku \citep{haiku2020github} (for the language model experiments).

\subsection{Unreported preliminary experiments}

During an exploration phase, we conducted a number of preliminary experiments, during
which we identified new statistics to collect, what hyperparameter
combinations to try, etc.  (For example, we wanted to minimize the
fraction of runs 
with learning rates too small to bring about the edge of stability, and
those with learning rates so large that training diverged.)
The results reported in this paper were one series of final runs
for the last combinations of hyperparameters.  

\section{Results}
\label{s:results}

All of the results from
every run that did not diverge may be found in a \href{https://drive.google.com/drive/folders/1mTgRsMO1Ebrnup3vwor5-Ih6zeC3-ljZ?usp=sharing}{supplementary folder}.  (In all of the plots, the training time in seconds is
plotted along the horizontal axis.)
In this section, we go over some of the most noteworthy results.

\subsection{MNIST}

Figure~\ref{f:mnist.rho=0.eigs} contains plots of
the magnitudes of the top three eigenvalues of the Hessian, along with $2/\eta$ and the
SAM-edge, when an MLP was trained on MNIST using gradient descent.
There is a plot for each learning rate $\eta$.
\begin{figure}
    \centering
    \begin{subfigure}{0.3\linewidth}
        \includegraphics[width=\linewidth]{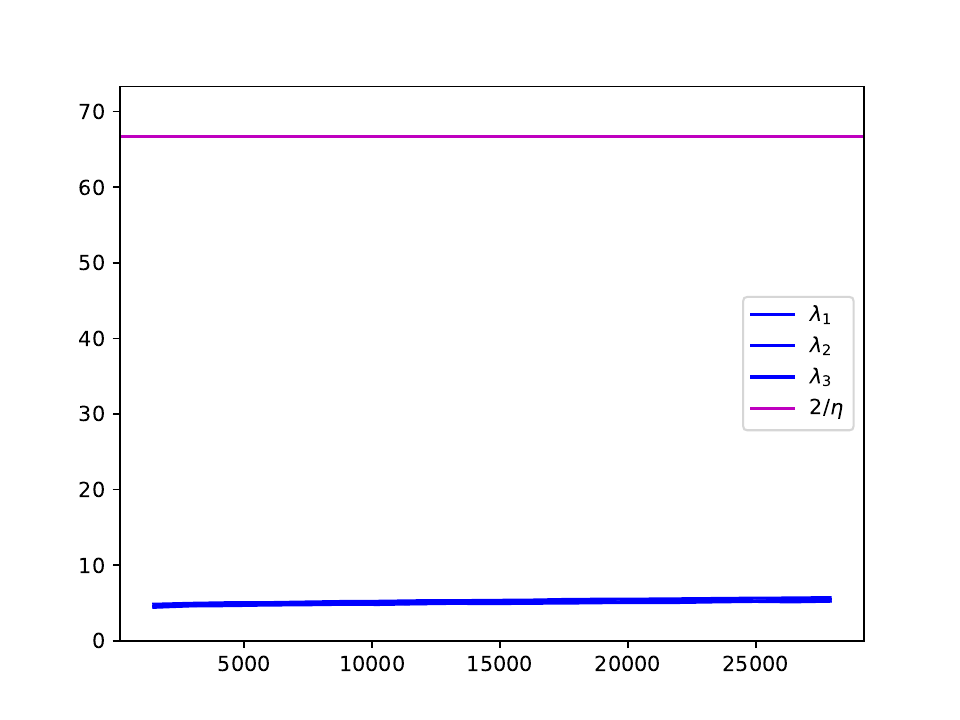}
        \caption{$\eta=0.03$}
    \end{subfigure}
    \begin{subfigure}{0.3\linewidth}
        \includegraphics[width=\linewidth]{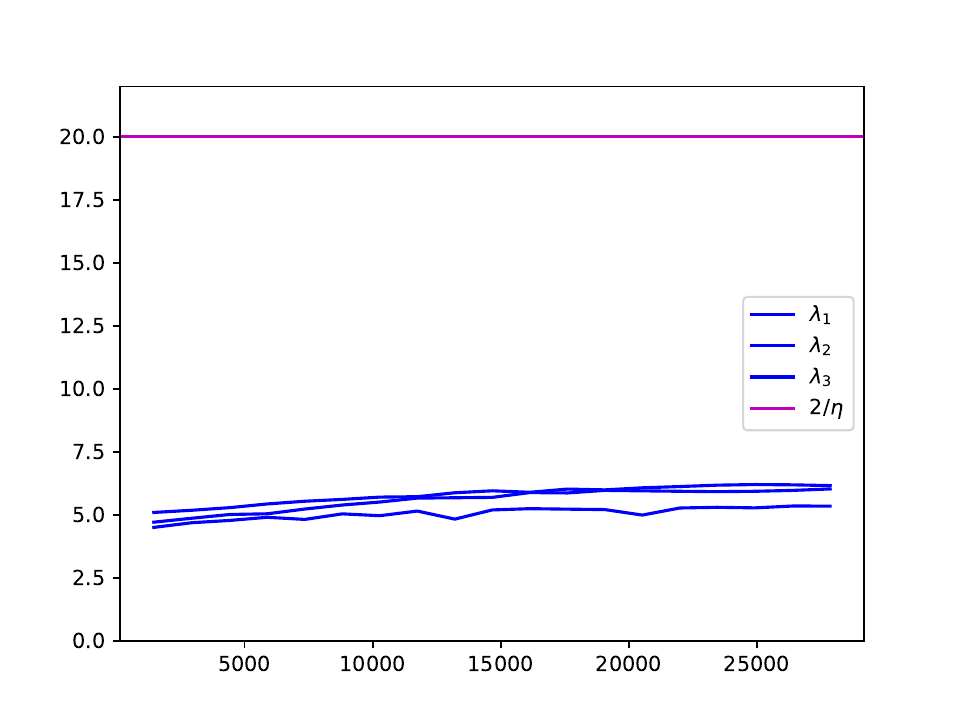}
        \caption{$\eta=0.1$}
    \end{subfigure}
    \begin{subfigure}{0.3\linewidth}
        \includegraphics[width=\linewidth]{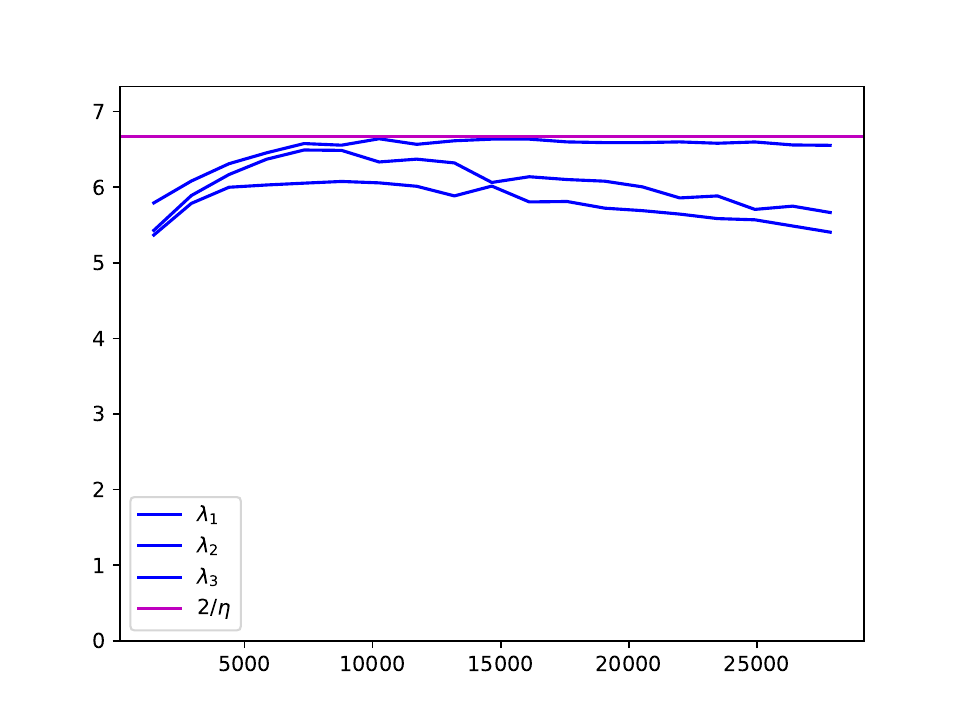}
        \caption{$\eta=0.3$}
    \end{subfigure}
    
    \caption{Magnitudes of the largest eigenvalues of the Hessian when an MLP is trained with GD on MNIST.}
    \label{f:mnist.rho=0.eigs}
\end{figure}
As
reported by \citet{cohen2021gradient}, if the learning rate is large enough, 
the operator norm of the Hessian stabilizes near $2/\eta$.
We can think of GD as a special case of SAM with $\rho = 0$;
the SAM-edge is of course $2/\eta$ in that case.

Figure~\ref{f:mnist.rho=0.1.eigs} contains the analogous plots
when $\rho = 0.1$.  
\begin{figure}
    \centering
    \begin{subfigure}{0.3\linewidth}
        \includegraphics[width=\linewidth]{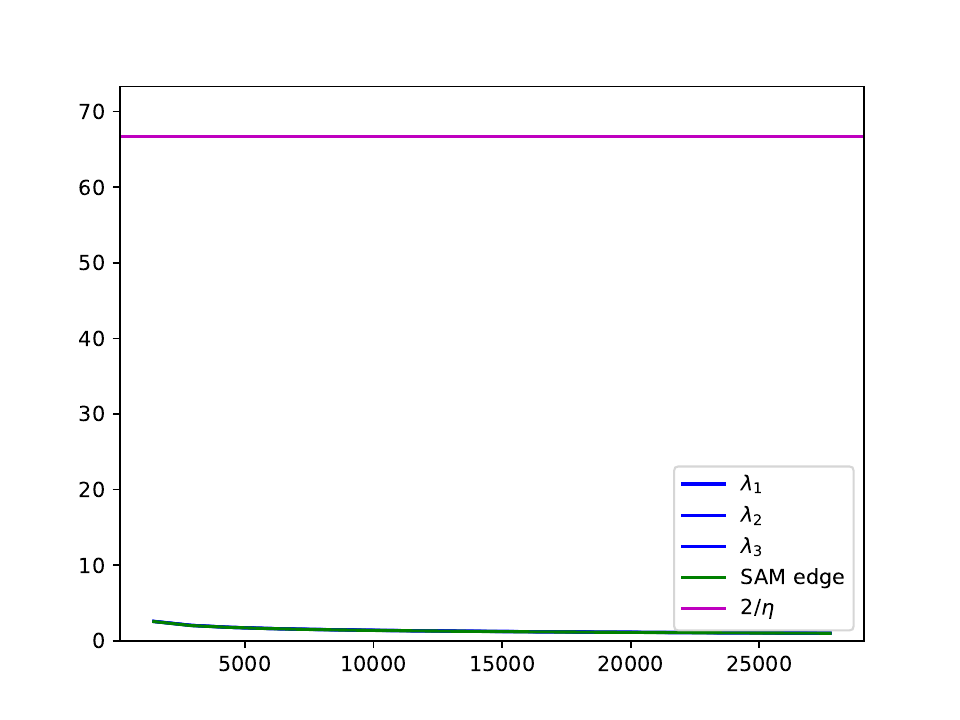}
        \caption{$\eta=0.03$}
    \end{subfigure}
    \begin{subfigure}{0.3\linewidth}
        \includegraphics[width=\linewidth]{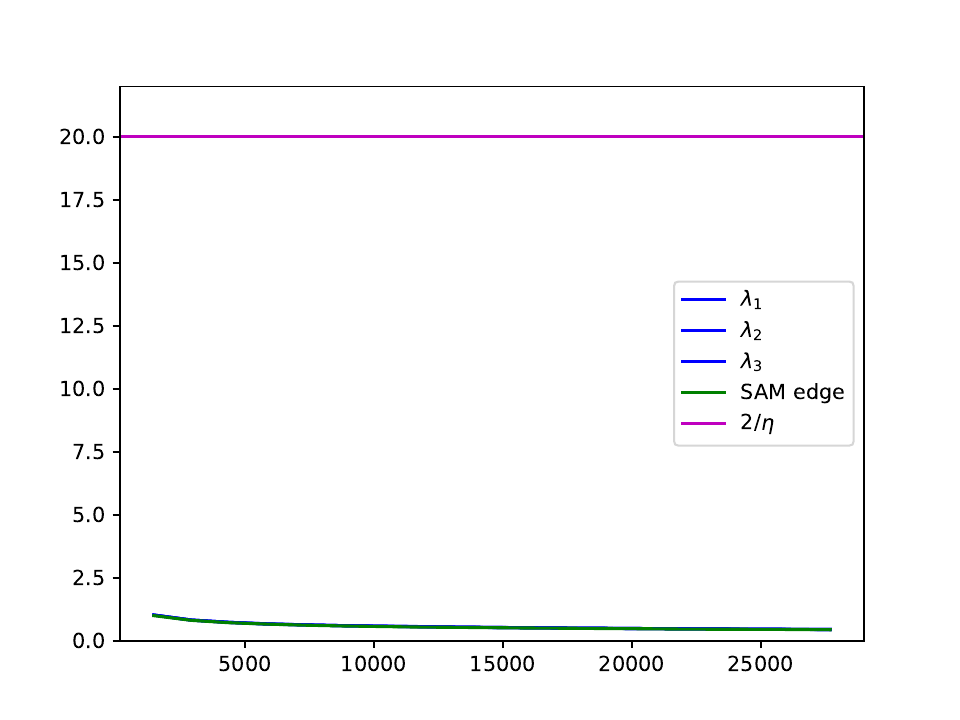}
        \caption{$\eta=0.1$}
    \end{subfigure}
    \begin{subfigure}{0.3\linewidth}
        \includegraphics[width=\linewidth]{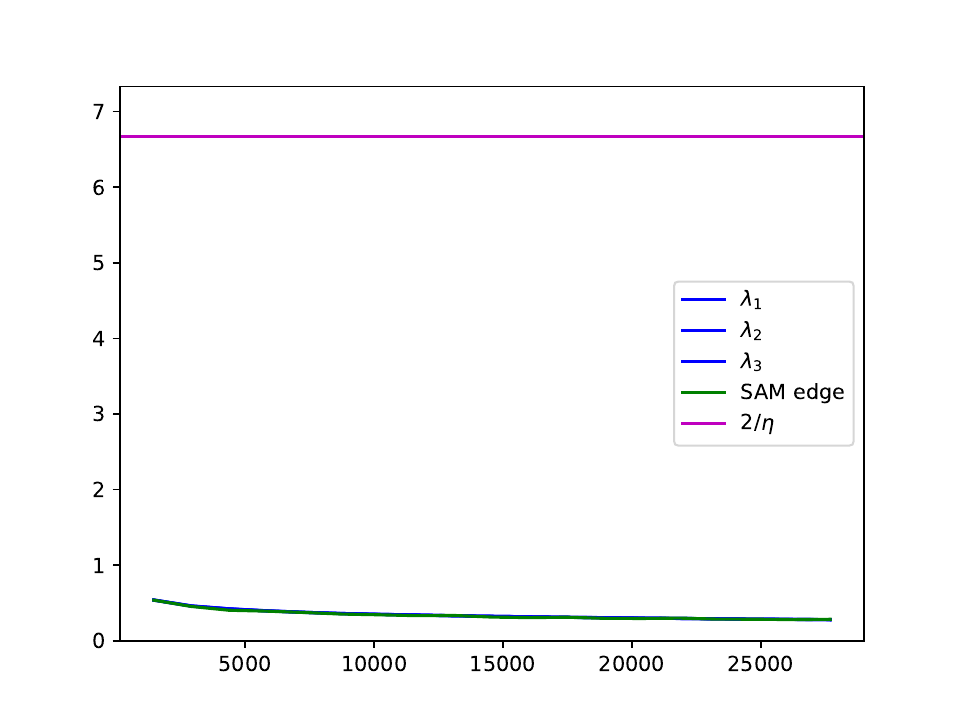}
        \caption{$\eta=0.3$}
    \end{subfigure}
    
    \caption{Magnitudes of the largest eigenvalues of the Hessian when an MLP is trained with SAM on MNIST, with $\rho=0.1$.}
    \label{f:mnist.rho=0.1.eigs}
\end{figure}
Despite the fact that gradients are taken from locations
at a distance just 0.1 from each of the iterates, the cumulative
effect results in solutions with Hessians an order of
magnitude smaller than those seen with GD.  

Figure~\ref{f:mnist.rho=0.1.eigs_no_gd_edge} contains the analogous plots, but without $2/\eta$, and with the axis
rescaled to zoom in on the SAM edge and the magnitudes of the principal
eigenvalues of the Hessian.  
\begin{figure}
    \centering
    \begin{subfigure}{0.3\linewidth}
        \includegraphics[width=\linewidth]{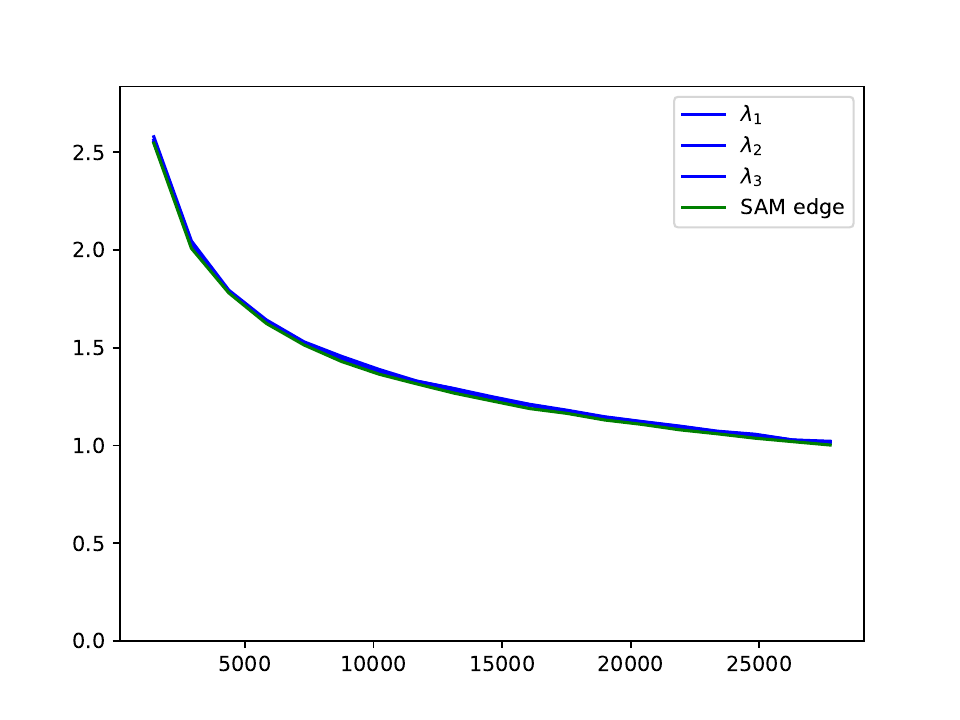}
        \caption{$\eta=0.03$}
    \end{subfigure}
    \begin{subfigure}{0.3\linewidth}
        \includegraphics[width=\linewidth]{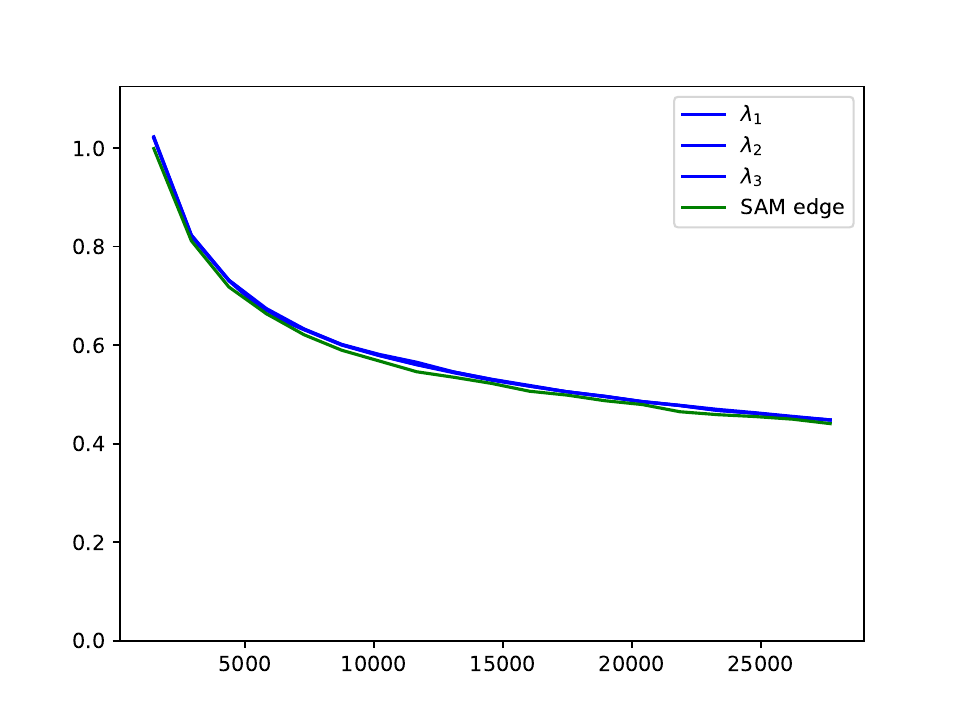}
        \caption{$\eta=0.1$}
    \end{subfigure}
    \begin{subfigure}{0.3\linewidth}
        \includegraphics[width=\linewidth]{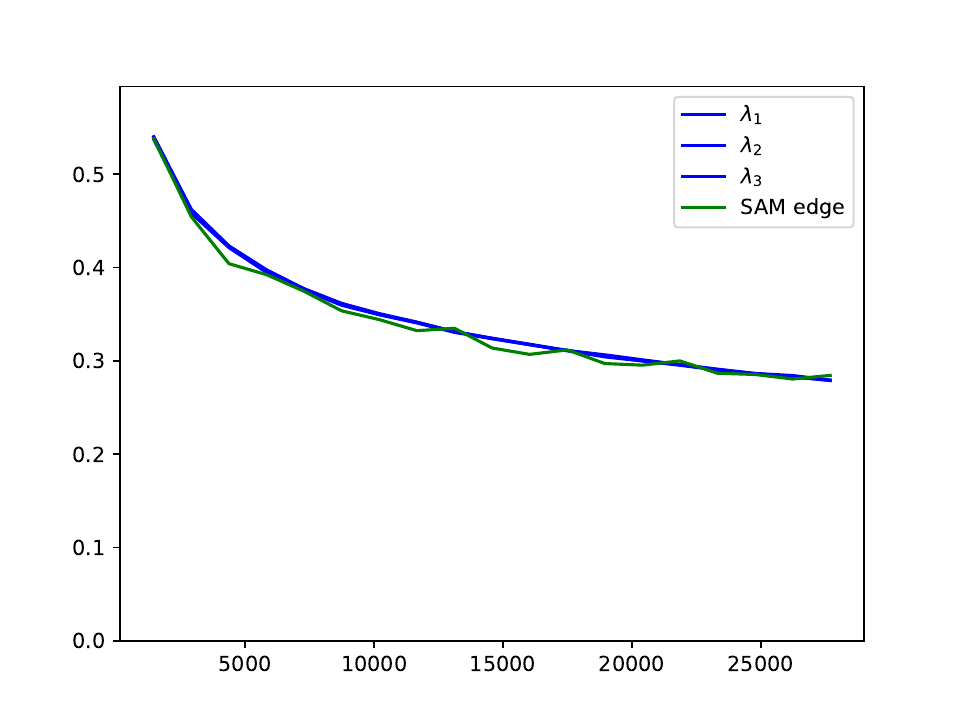}
        \caption{$\eta=0.3$}
    \end{subfigure}
    \caption{Magnitudes of the largest eigenvalues of the Hessian when an MLP is trained with SAM on MNIST, with $\rho=0.1$.}
    \label{f:mnist.rho=0.1.eigs_no_gd_edge}
\end{figure}
The operator norm closely tracks the
SAM edge derived in Section~\ref{s:derivation}. 
SAM operates at the edge of stability for a wider variety
of learning rates than GD.
We also see the SAM edge decreasing over time, as
the gradients get smaller.  The top three principal
components are very close to one another.   This
is consistent with the view that SAM effectively
performs gradient descent on the operator norm of
the Hessian -- if it did, a step would reduce the
principal eigenvalue, while leaving the others at their
old values, bringing the top eigenvalue closer to
the others.

In Figure~\ref{f:mnist.loss}, we plot the training
losses, 
when $\rho = 0.0$ and $\rho = 0.1$.
\begin{figure}
    \centering
    \begin{subfigure}{0.3\linewidth}
        \includegraphics[width=\linewidth]{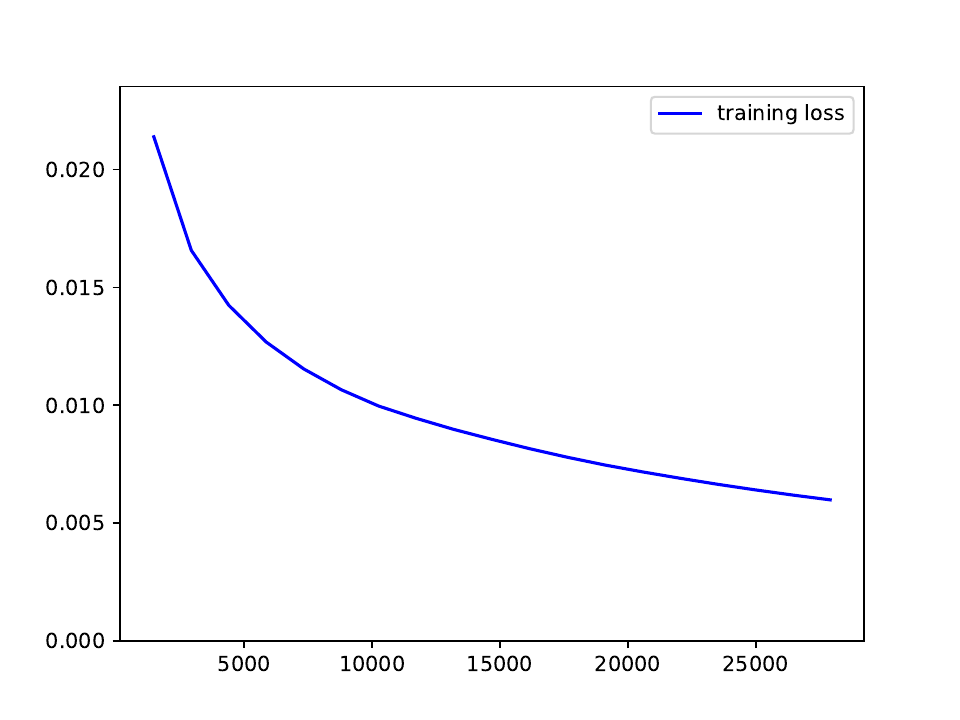}
        \caption{$\eta=0.03, \rho = 0$}
    \end{subfigure}
    \begin{subfigure}{0.3\linewidth}
        \includegraphics[width=\linewidth]{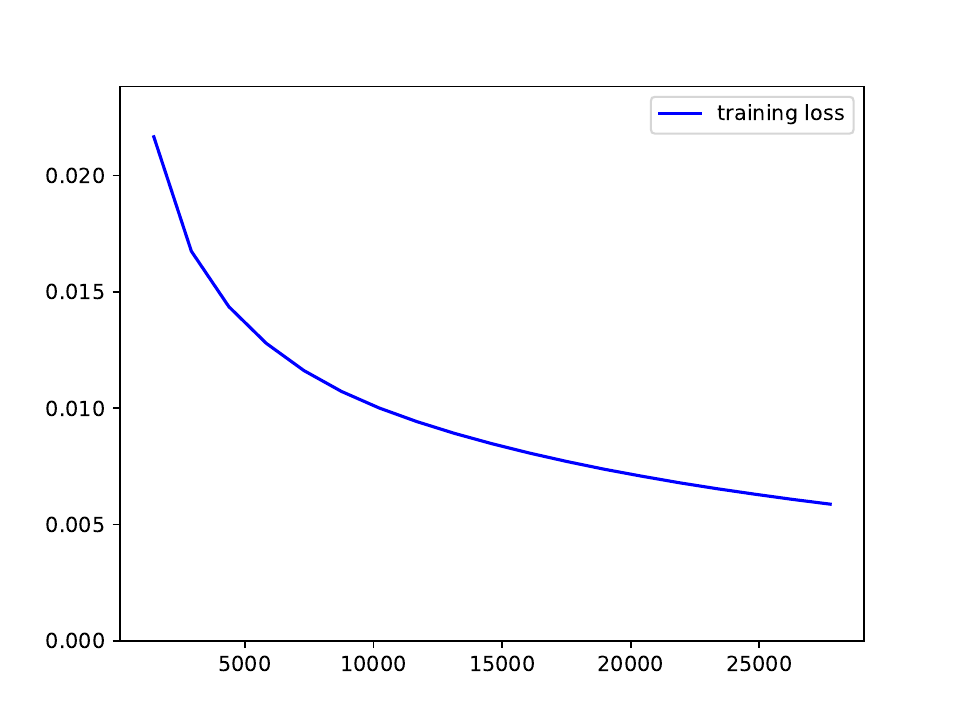}
        \caption{$\eta=0.03, \rho=0.1$}
    \end{subfigure}

    \begin{subfigure}{0.3\linewidth}
        \includegraphics[width=\linewidth]{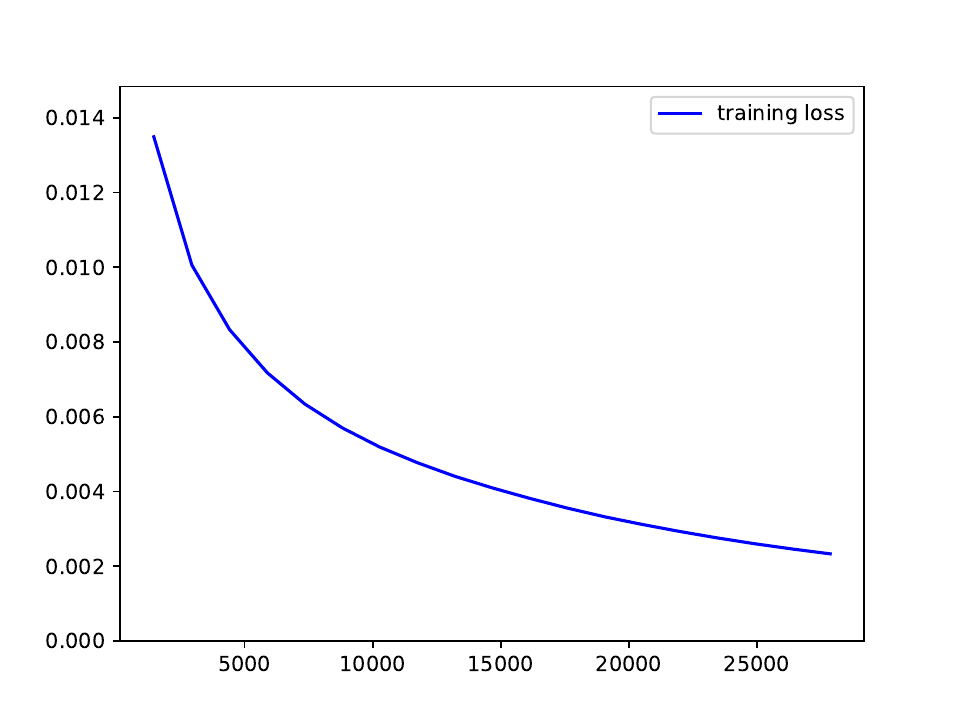}
        \caption{$\eta=0.1, \rho=0$}
    \end{subfigure}
    \begin{subfigure}{0.3\linewidth}
        \includegraphics[width=\linewidth]{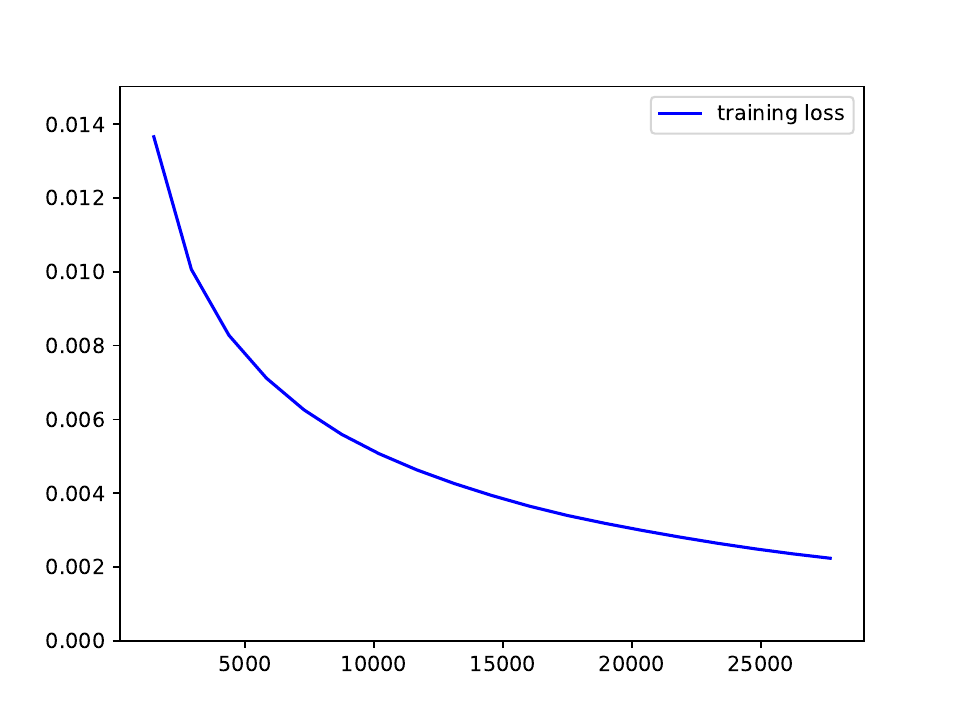}
        \caption{$\eta=0.1, \rho=0.1$}
    \end{subfigure}

    \begin{subfigure}{0.3\linewidth}
        \includegraphics[width=\linewidth]{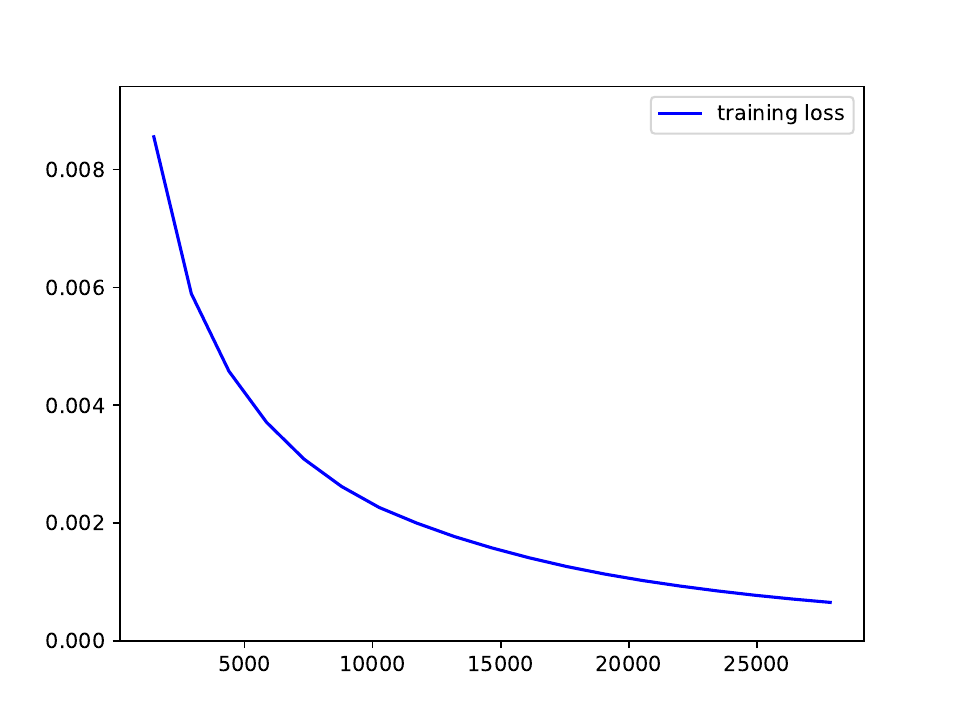}
        \caption{$\eta=0.3, \rho=0$}
    \end{subfigure}
    \begin{subfigure}{0.3\linewidth}
        \includegraphics[width=\linewidth]{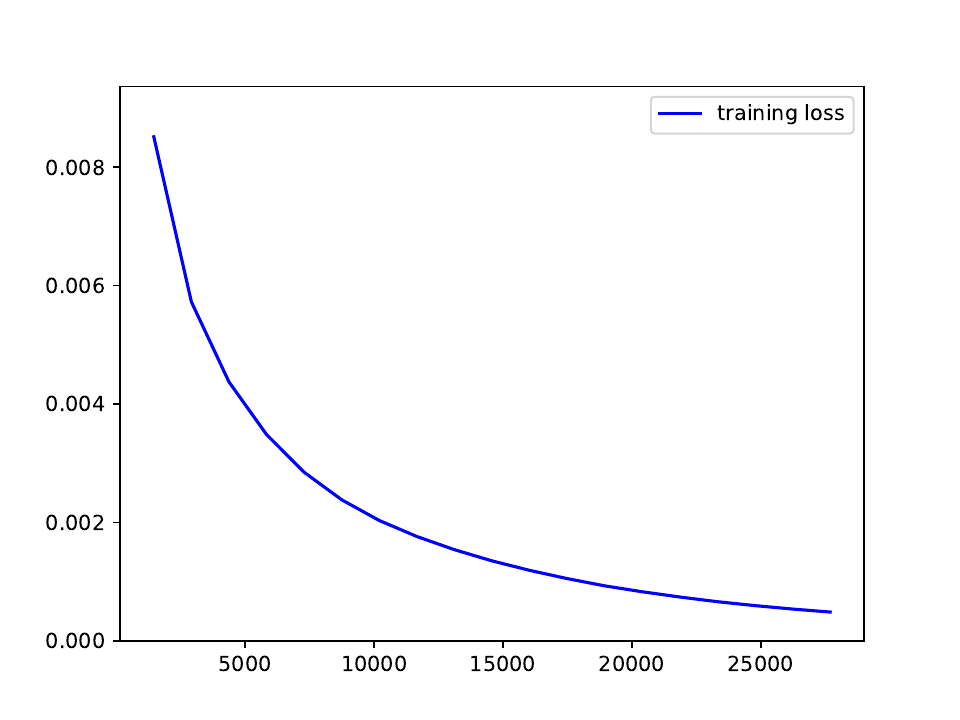}
        \caption{$\eta=0.3, \rho=0.1$}
    \end{subfigure}
    
    \caption{Training loss with GD and SAM on MNIST.}
    \label{f:mnist.loss}
\end{figure}
SAM achieves flatter minima with similar loss.
We also see that SAM drives training toward smoother
regions in parameter space while the training error is
still fairly high.

In Figure~\ref{f:mnist.alignment}, we examine alignments
between the gradients and the principal eigenvector of
the Hessian, again where $\rho = 0.1$. We evaluate
both the gradient at the iterate, and the gradient evaluated
by SAM, at a distance $\rho$ uphill.
\begin{figure}
    \centering
    \begin{subfigure}{0.3\linewidth}
        \includegraphics[width=\linewidth]{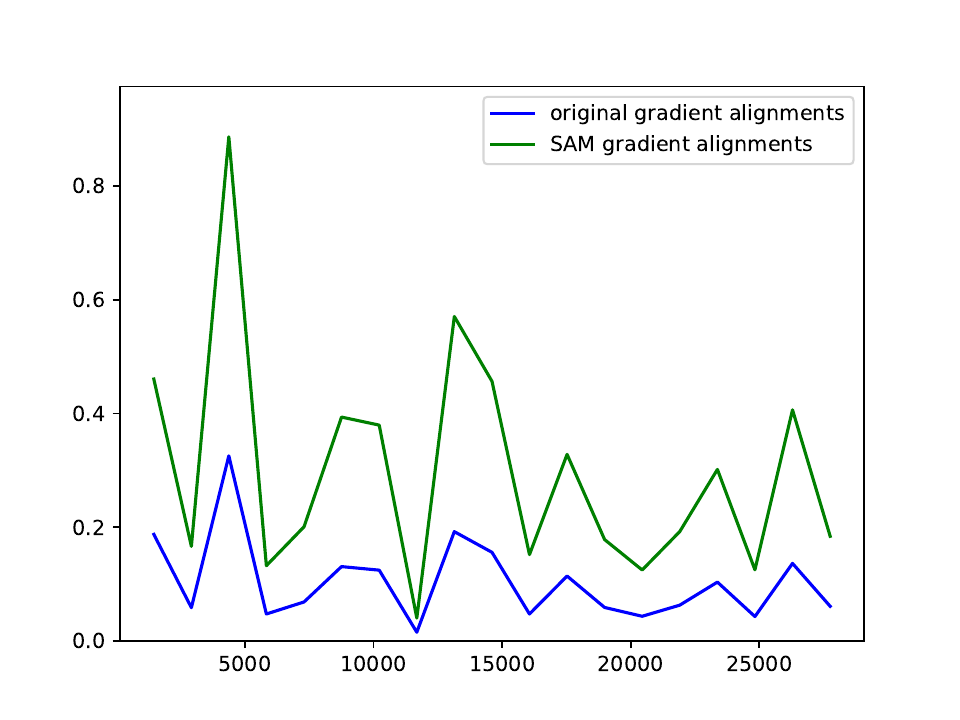}
        \caption{$\eta=0.03$}
    \end{subfigure}
    \begin{subfigure}{0.3\linewidth}
        \includegraphics[width=\linewidth]{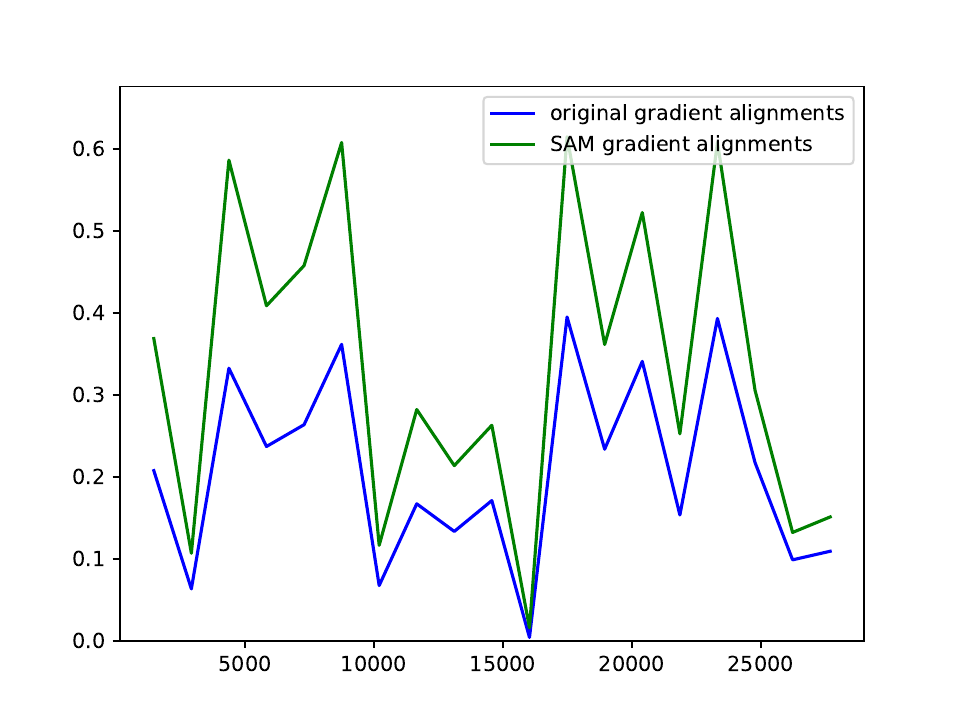}
        \caption{$\eta=0.1$}
    \end{subfigure}
    \begin{subfigure}{0.3\linewidth}
        \includegraphics[width=\linewidth]{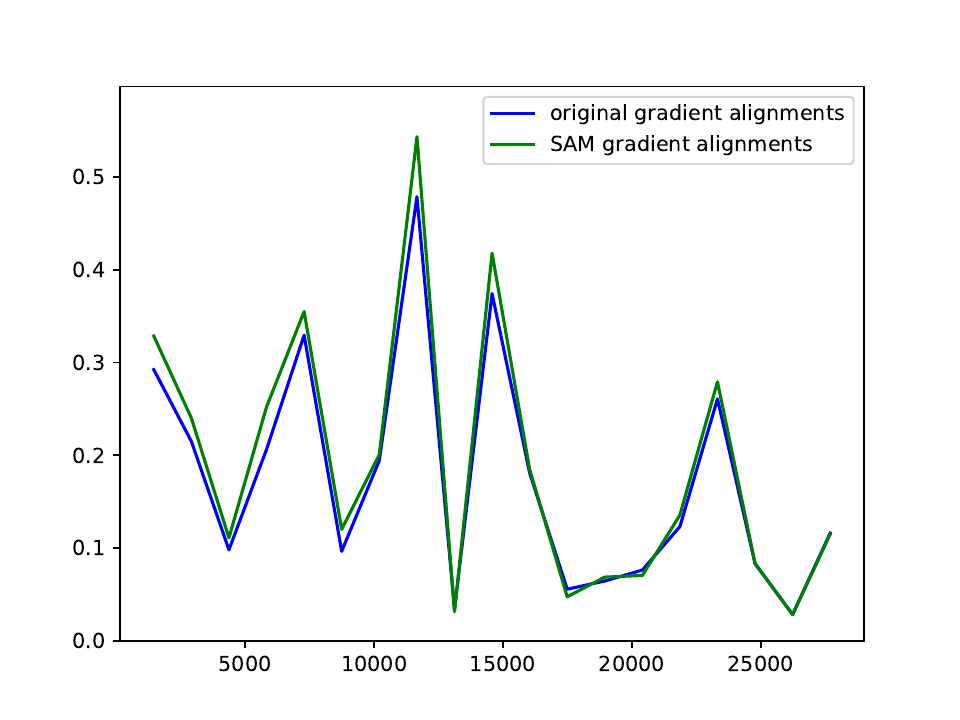}
        \caption{$\eta=0.3$}
    \end{subfigure}
    
    \caption{Alignments between gradients and the principal eigenvector
    of the Hessian with SAM on MNIST when $\rho = 0.1$.}
    \label{f:mnist.alignment}
\end{figure}
Since there are millions of parameters, random
directions would have a tiny amount of alignment.
We see a significant alignment between both gradients
and the principal eigenvector of the Hessian, though
the gradient used by SAM is aligned more closely.  Recall that there
are a number of eigenvectors whose eigenvalues are
nearly equal to the largest value.  Reducing their eigenvalues
can also make progress toward ultimately reducing the operator
norm of the Hessian.

\subsection{CIFAR10}
\label{s:cifar10}

In this section, we report on experiments
with convolutional neural
networks trained 
on 1000 examples from CIFAR10.  

As before, we start with the
case of GD in Figure~\ref{f:cifar10.rho=0.eigs}.
\begin{figure}
    \centering
    \begin{subfigure}{0.3\linewidth}
        \includegraphics[width=\linewidth]{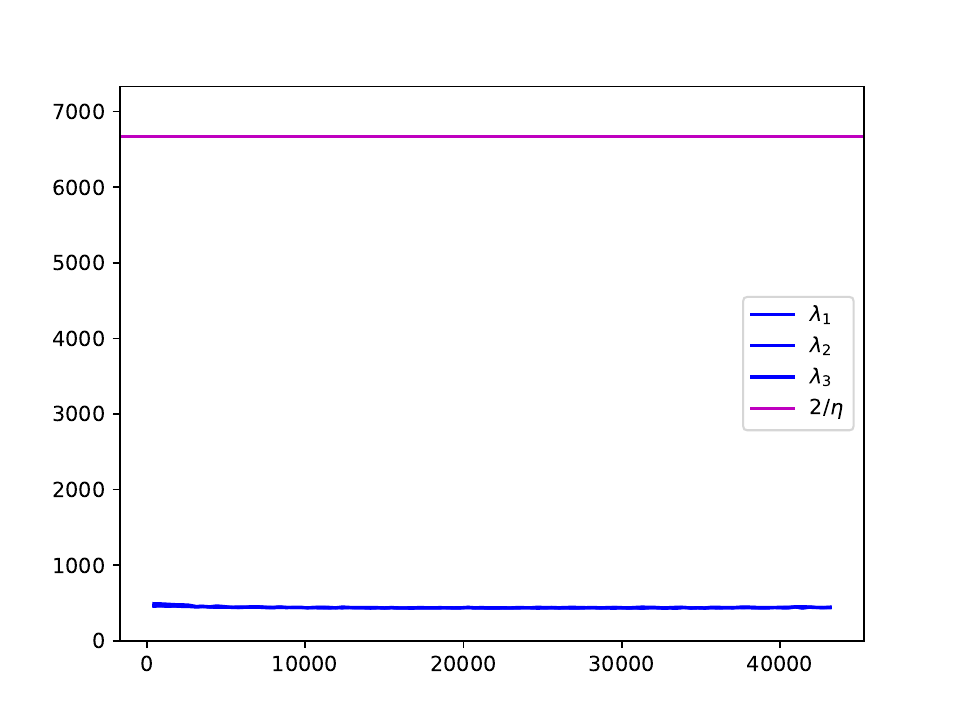}
        \caption{$\eta=0.0003$}
    \end{subfigure} 
    \begin{subfigure}{0.3\linewidth}
        \includegraphics[width=\linewidth]{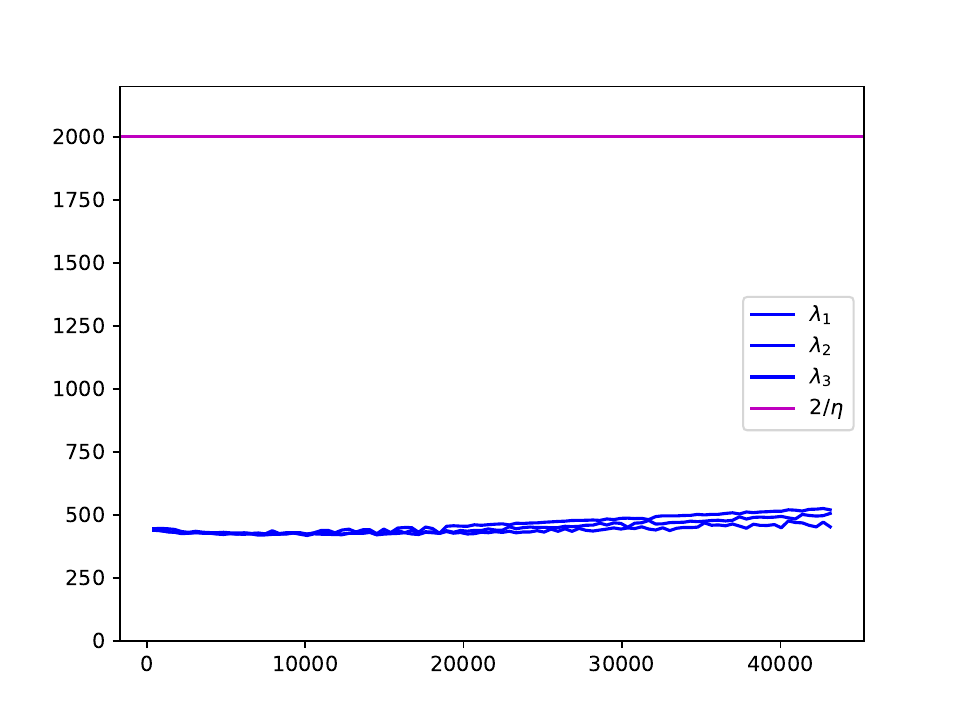}
        \caption{$\eta=0.001$}
    \end{subfigure} \\
    \begin{subfigure}{0.3\linewidth}
        \includegraphics[width=\linewidth]{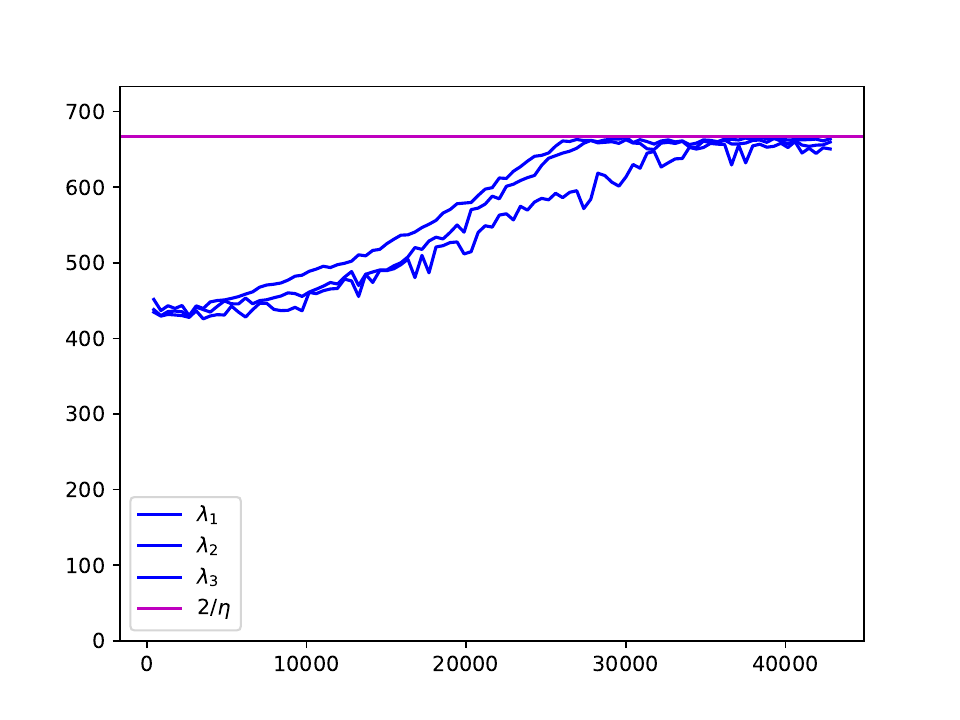}
        \caption{$\eta=0.003$}
    \end{subfigure}
    \begin{subfigure}{0.3\linewidth}
        \includegraphics[width=\linewidth]{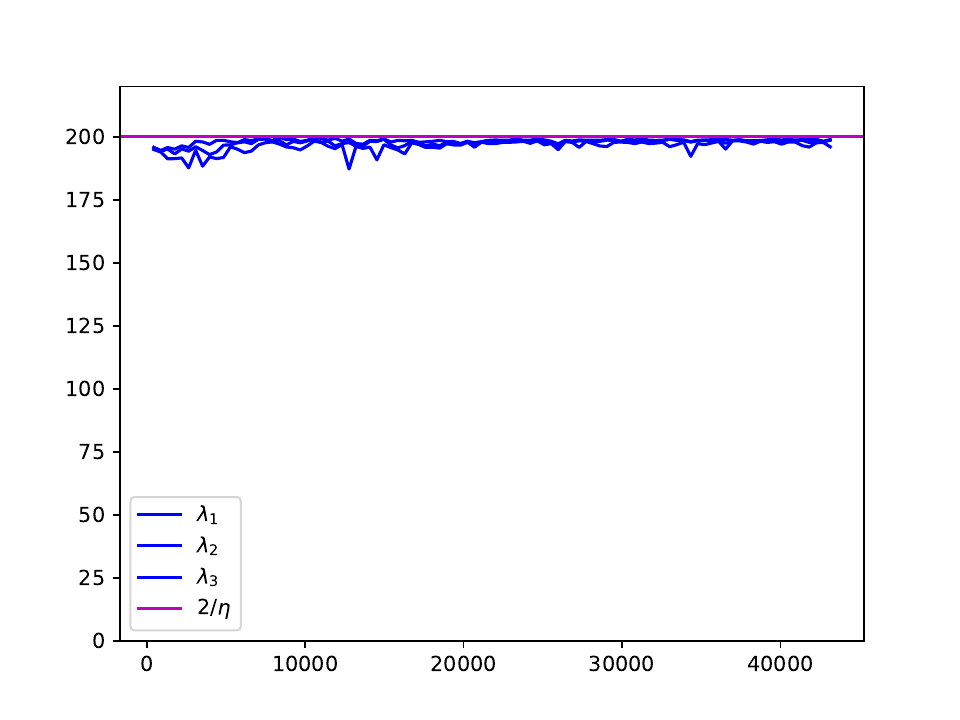}
        \caption{$\eta=0.01$}
    \end{subfigure}
    
    \caption{Magnitudes of the largest eigenvalues of the Hessian when a CNN is trained with GD on 1000 examples
    from CIFAR10.}
    \label{f:cifar10.rho=0.eigs}
\end{figure}
At the larger learning rates, training is reaching the edge of stability.

Next, we plot the same quantities when the
network is trained with SAM, with $\rho = 0.1$,
in Figure~\ref{f:cifar10.rho=0.1.eigs}.
\begin{figure}
    \centering
    \begin{subfigure}{0.3\linewidth}
        \includegraphics[width=\linewidth]{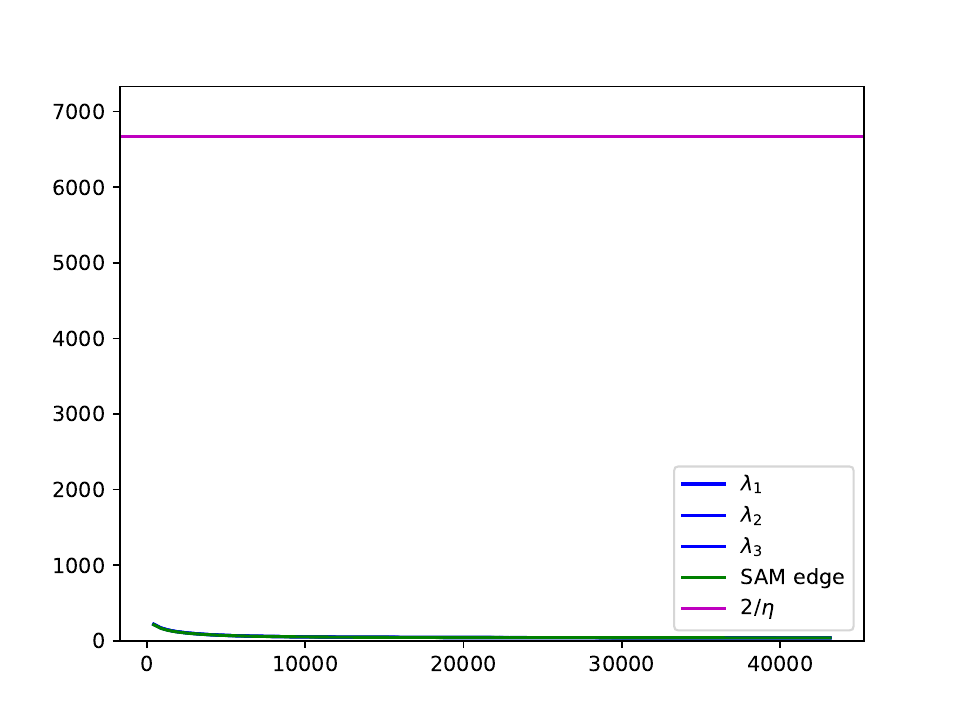}
        \caption{$\eta=0.0003$}
    \end{subfigure}
    \begin{subfigure}{0.3\linewidth}
        \includegraphics[width=\linewidth]{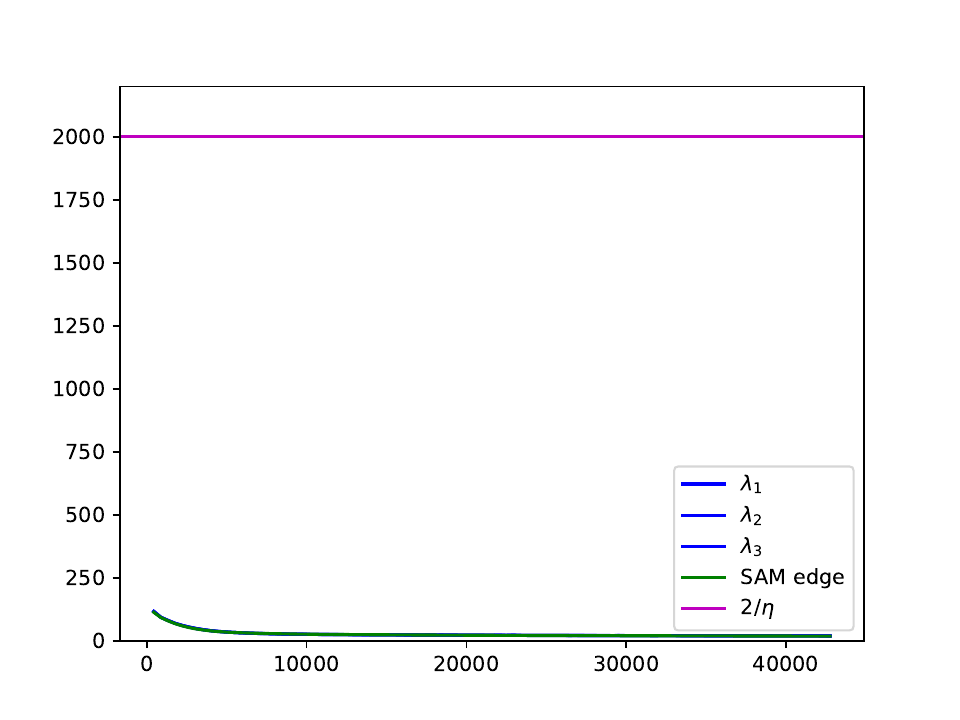}
        \caption{$\eta=0.001$}
    \end{subfigure}
    \begin{subfigure}{0.3\linewidth}
        \includegraphics[width=\linewidth]{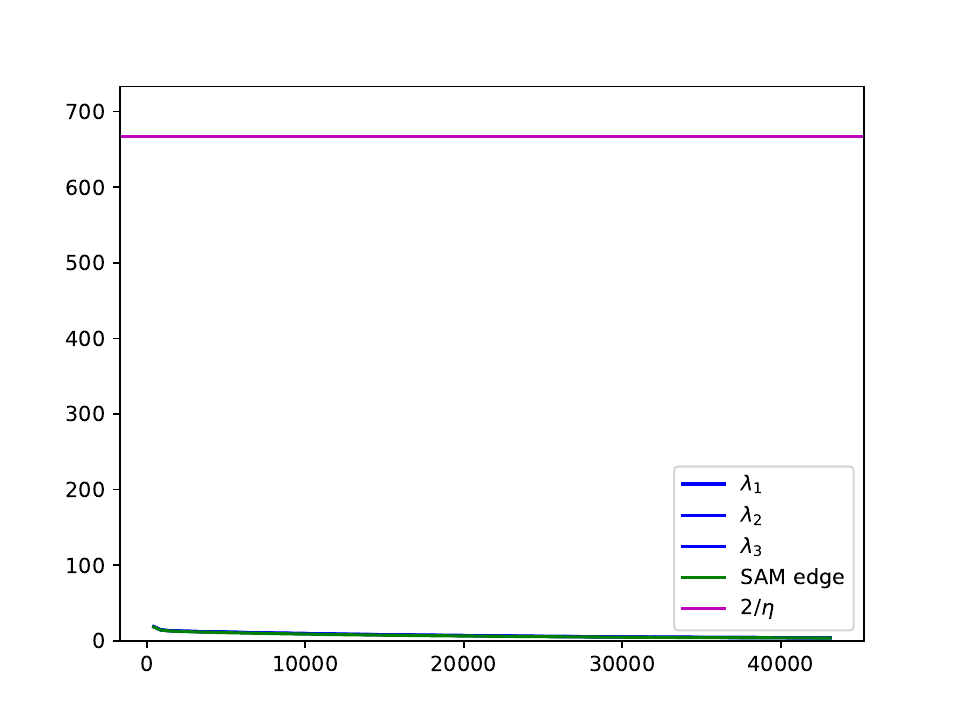}
        \caption{$\eta=0.003$}
    \end{subfigure}
    
    \caption{Magnitudes of the largest eigenvalues of the Hessian when a CNN is trained with SAM,
    with $\rho = 0.1$, on CIFAR10.}
    \label{f:cifar10.rho=0.1.eigs}
\end{figure}
Here, the eigenvalues are multiple orders of magnitude smaller than
$2/\eta$.

Next, 
in Figure~\ref{f:cifar10.rho=0.1.eigs_no_two_over_eta} we no longer plot $2/\eta$,
and zoom in on the region where the SAM
edge and the eigenvalues are.  
\begin{figure}
    \centering
   \begin{subfigure}{0.3\linewidth}
        \includegraphics[width=\linewidth]{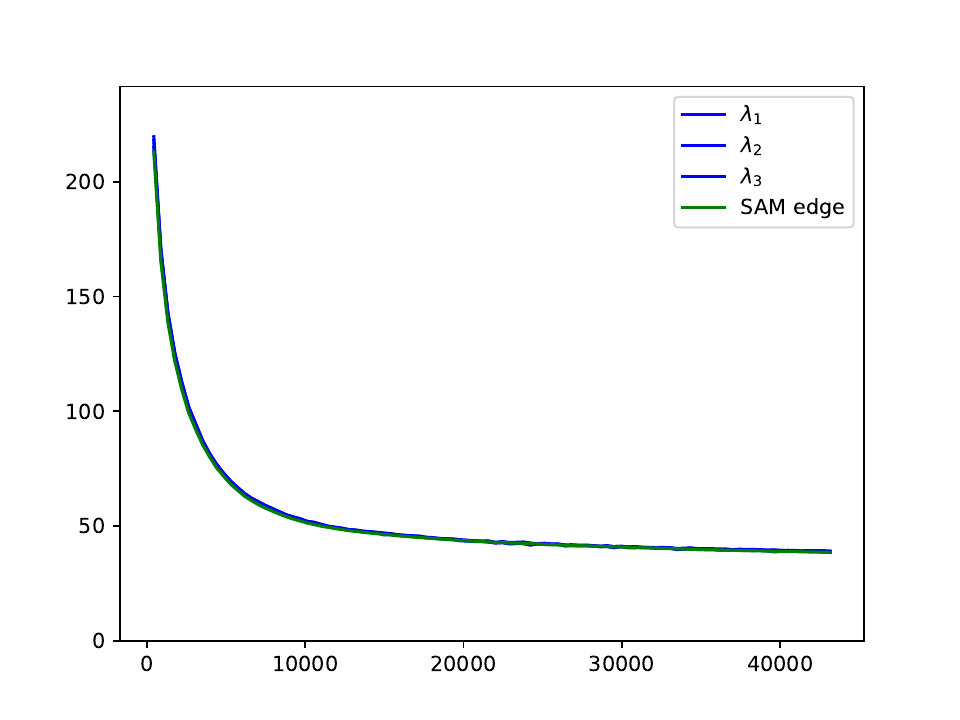}
        \caption{$\eta=0.0003$}
    \end{subfigure}
    \begin{subfigure}{0.3\linewidth}
        \includegraphics[width=\linewidth]{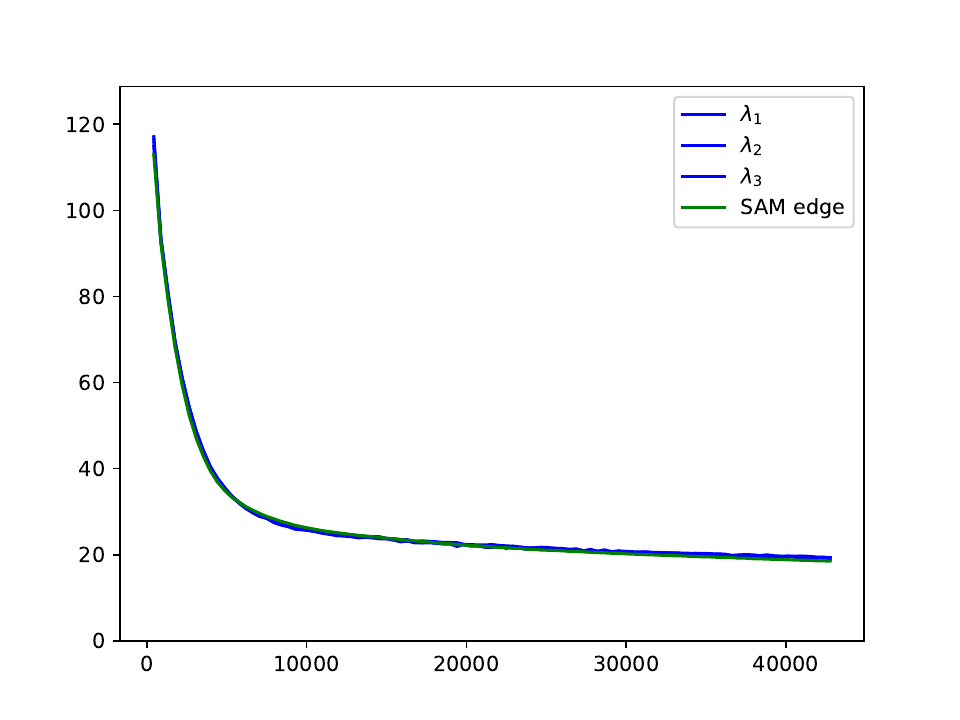}
        \caption{$\eta=0.001$}
    \end{subfigure}
    \begin{subfigure}{0.3\linewidth}
        \includegraphics[width=\linewidth]{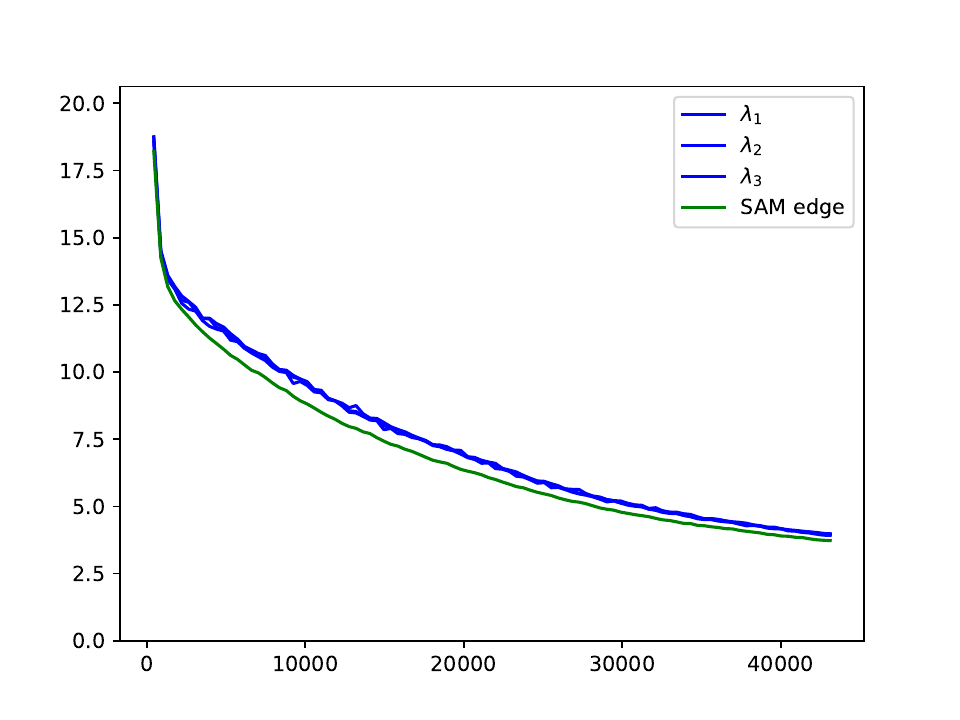}
        \caption{$\eta=0.003$}
    \end{subfigure}
    \caption{Magnitudes of the largest eigenvalues of the Hessian when a CNN is trained with SAM, with $\rho = 0.1$, on CIFAR10.}
    \label{f:cifar10.rho=0.1.eigs_no_two_over_eta}
\end{figure}
Here, as with MNIST, we once again see SAM operating at the
edge of stability identified in Section~\ref{s:derivation},
even at learning rates where GD did not.

Figure~\ref{f:cifar10.loss} contains plots of
the training loss on CIFAR10, for $\rho = 0.0$ and
$\rho = 0.1$.
\begin{figure}
    \centering   
      \begin{subfigure}{0.3\linewidth}
        \includegraphics[width=\linewidth]{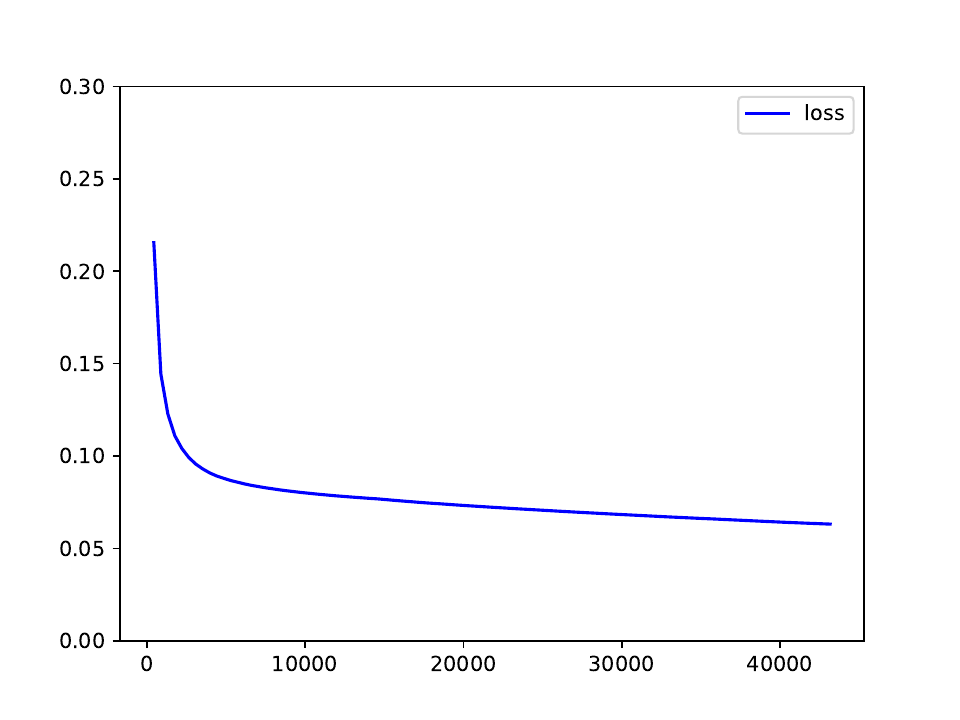}
        \caption{GD, $\eta=0.0003$}
    \end{subfigure}
    \begin{subfigure}{0.3\linewidth}
        \includegraphics[width=\linewidth]{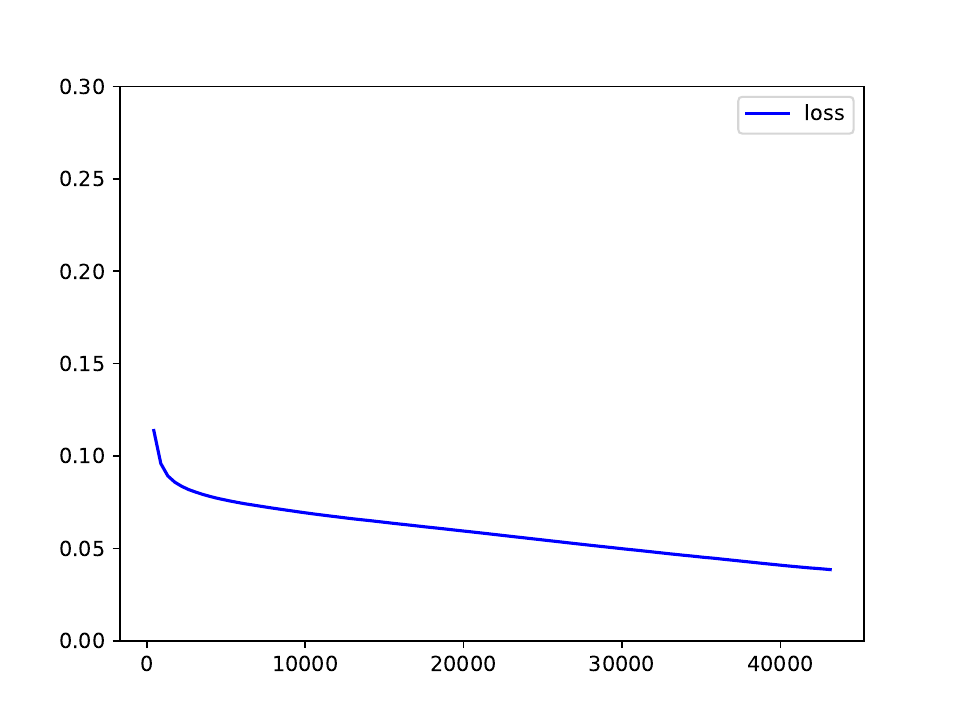}
        \caption{GD, $\eta=0.001$}
    \end{subfigure}
    \begin{subfigure}{0.3\linewidth}
        \includegraphics[width=\linewidth]{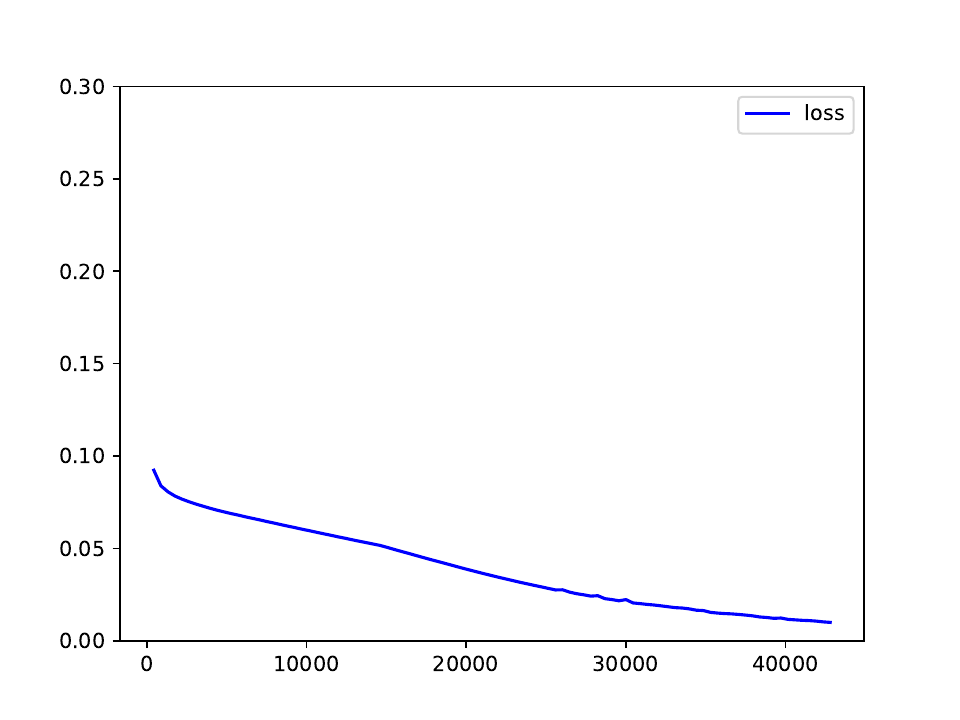}
        \caption{GD, $\eta=0.003$}
    \end{subfigure}
    \begin{subfigure}{0.3\linewidth}
        \includegraphics[width=\linewidth]{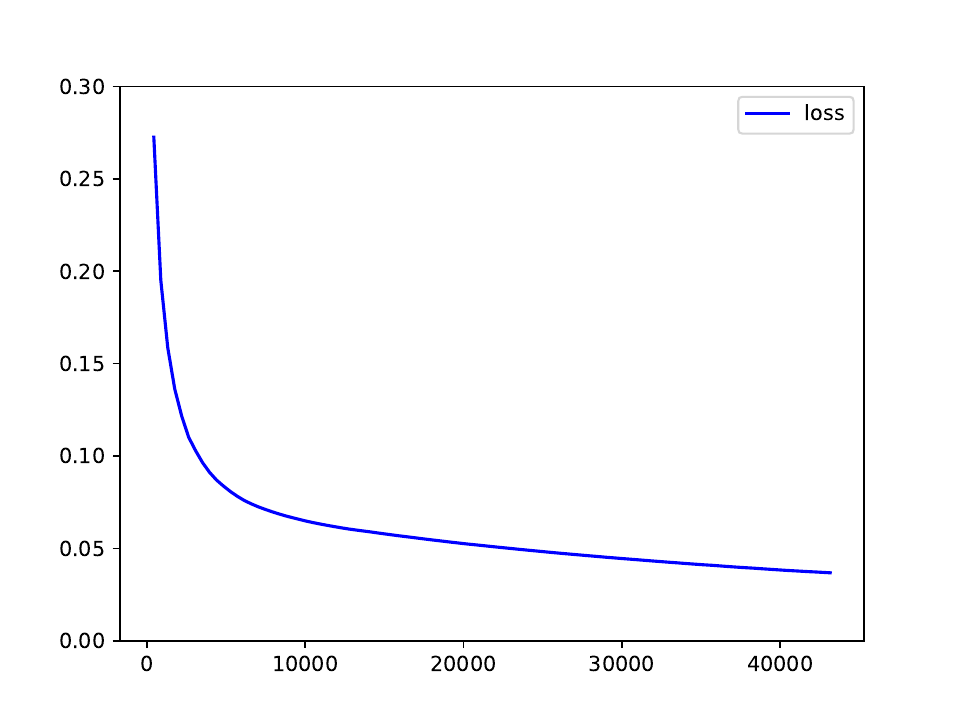}
        \caption{SAM, $\eta=0.0003$}
    \end{subfigure}
    \begin{subfigure}{0.3\linewidth}
        \includegraphics[width=\linewidth]{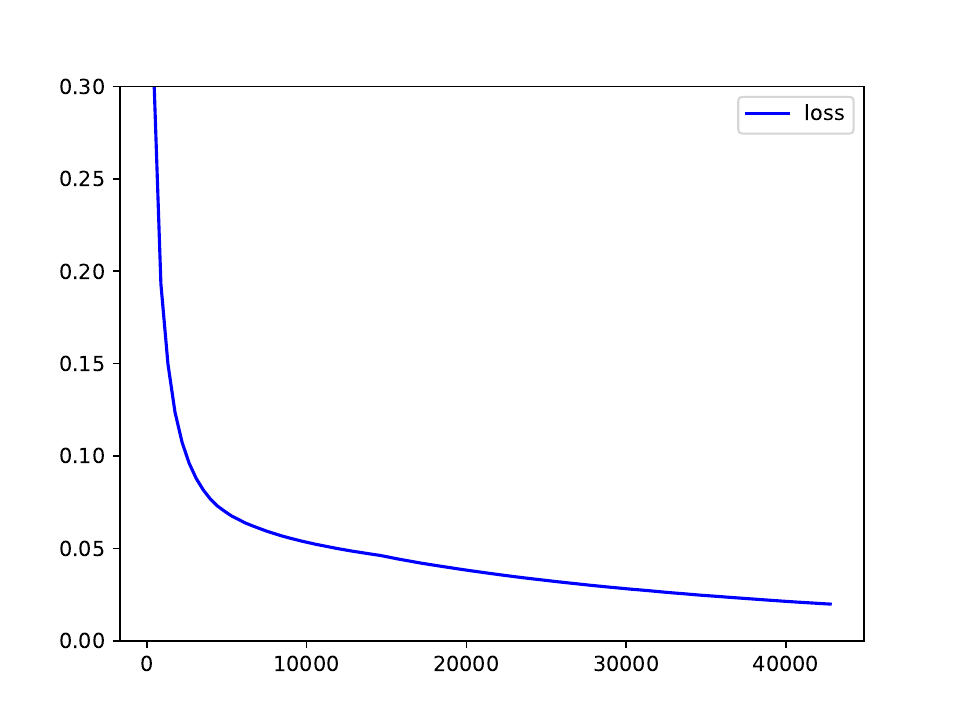}
        \caption{SAM, $\eta=0.001$}
    \end{subfigure}
    \begin{subfigure}{0.3\linewidth}
        \includegraphics[width=\linewidth]{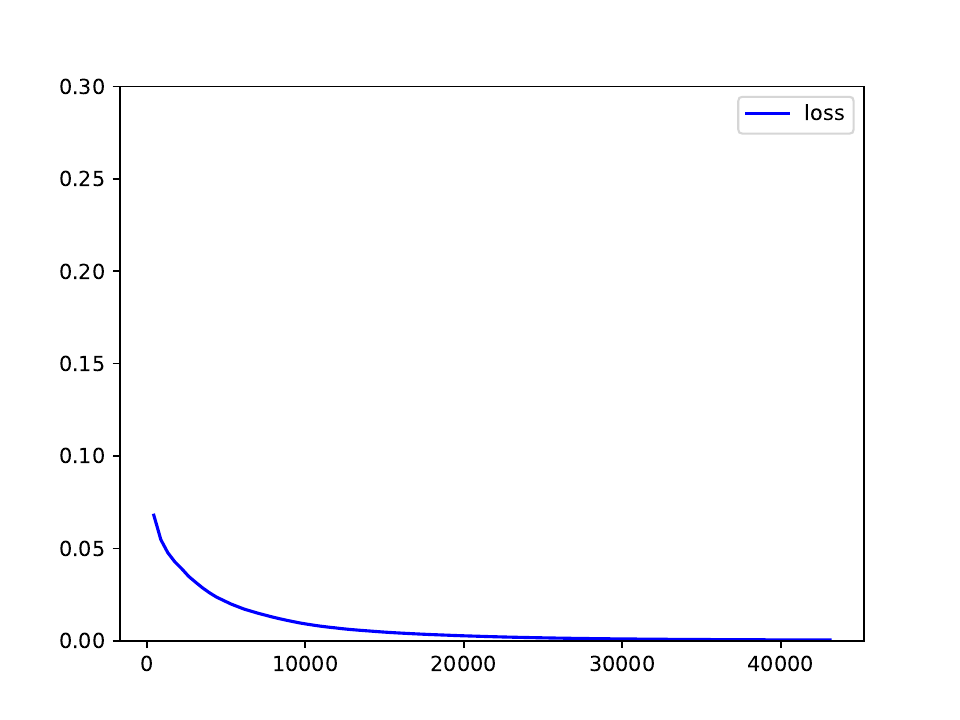}
        \caption{SAM, $\eta=0.003$}
    \end{subfigure}
    \caption{Training loss with SGD and SAM (with $\rho = 0.1$) on CIFAR10.}
    \label{f:cifar10.loss}
\end{figure}
In this task, SAM achieves wider minima without sacrificing training error.
In fact, for the larger step sizes, its training error is better.

In Figure~\ref{f:cifar10.alignment}, we examine alignments
between the gradients and the principal eigenvector of
the Hessian in the case where $\rho = 0.1$ and
a CNN is trained on CIFAR10.
\begin{figure}
    \centering
   \begin{subfigure}{0.3\linewidth}
        \includegraphics[width=\linewidth]{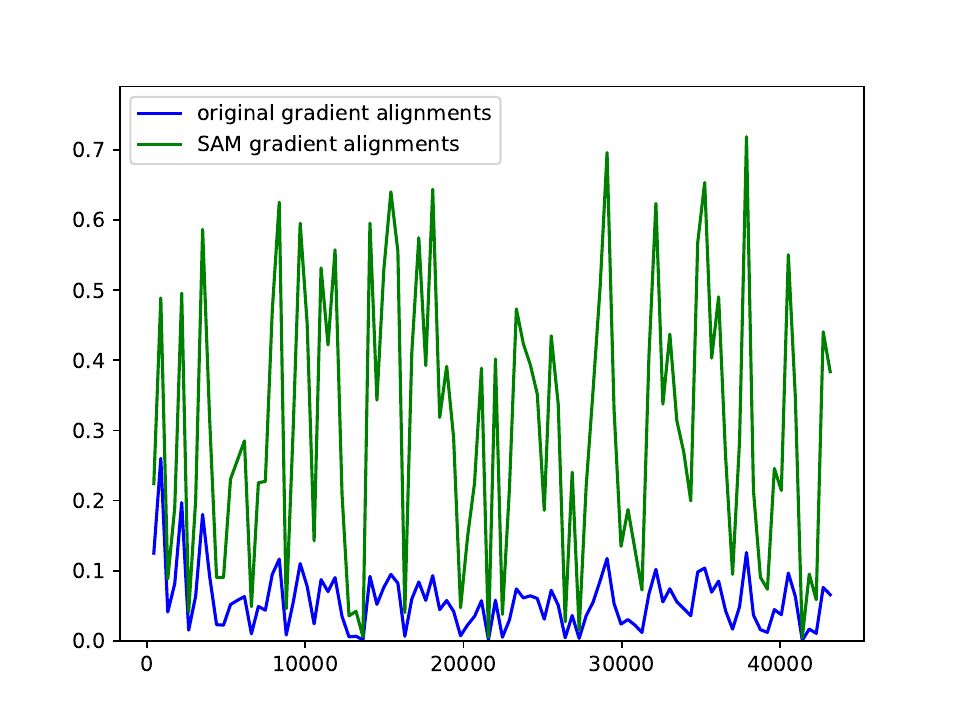}
        \caption{$\eta=0.0003$}
    \end{subfigure}
    \begin{subfigure}{0.3\linewidth}
        \includegraphics[width=\linewidth]{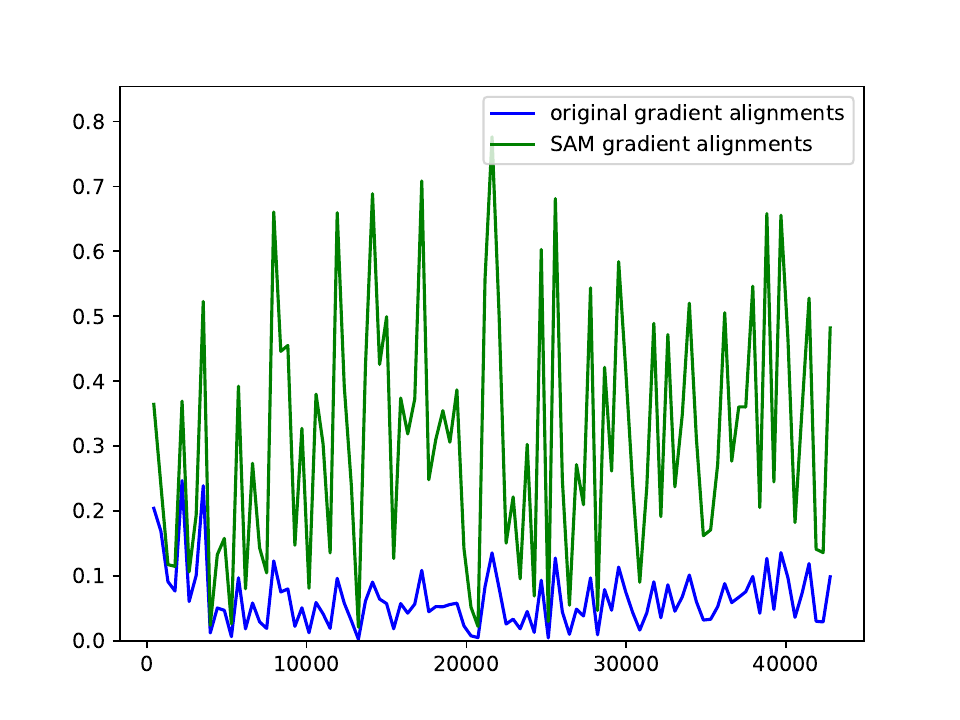}
        \caption{$\eta=0.001$}
    \end{subfigure}
    \begin{subfigure}{0.3\linewidth}
        \includegraphics[width=\linewidth]{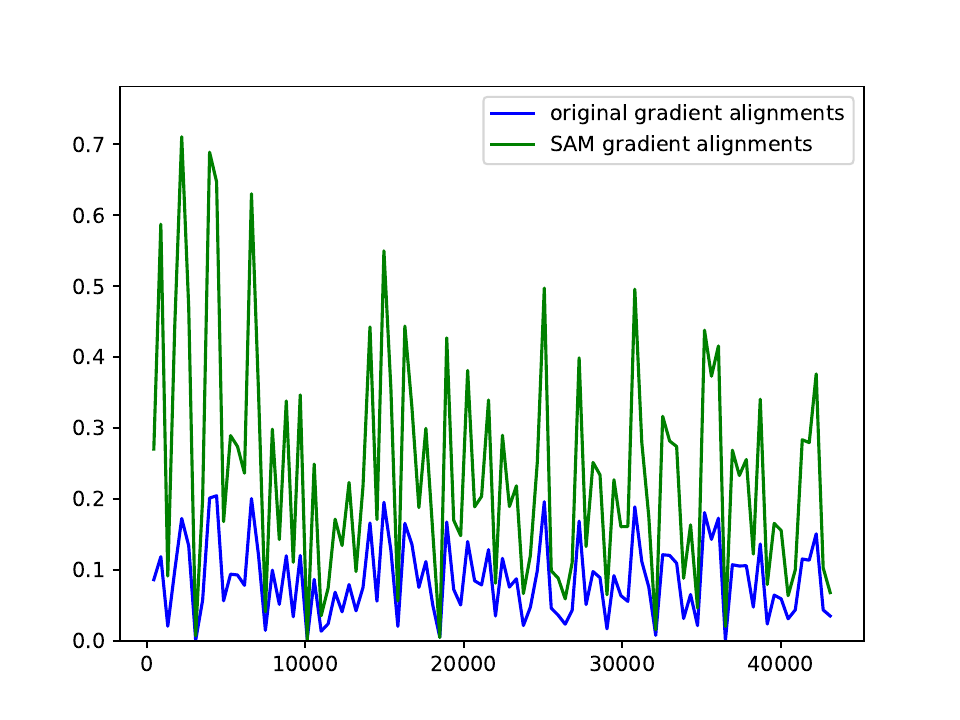}
        \caption{$\eta=0.003$}
    \end{subfigure}    
    \caption{Alignments between gradients and the principal eigenvector
    of the Hessian with SAM on CIFAR10.}
    \label{f:cifar10.alignment}
\end{figure}
Again, we see significant alignment, especially at the higher 
learning rates.  As in MNIST, we also see stronger alignment
with the principal direction
for the gradients evaluated at the uphill location used by
SAM.

\subsection{Language modeling}
\label{s:lm}

Next, we report on experiments training a language model.
As before, we start with SGD, here in Figure~\ref{f:lm.rho=0.eigs}.
\begin{figure}
    \centering
    \begin{subfigure}{0.3\linewidth}
        \includegraphics[width=\linewidth]{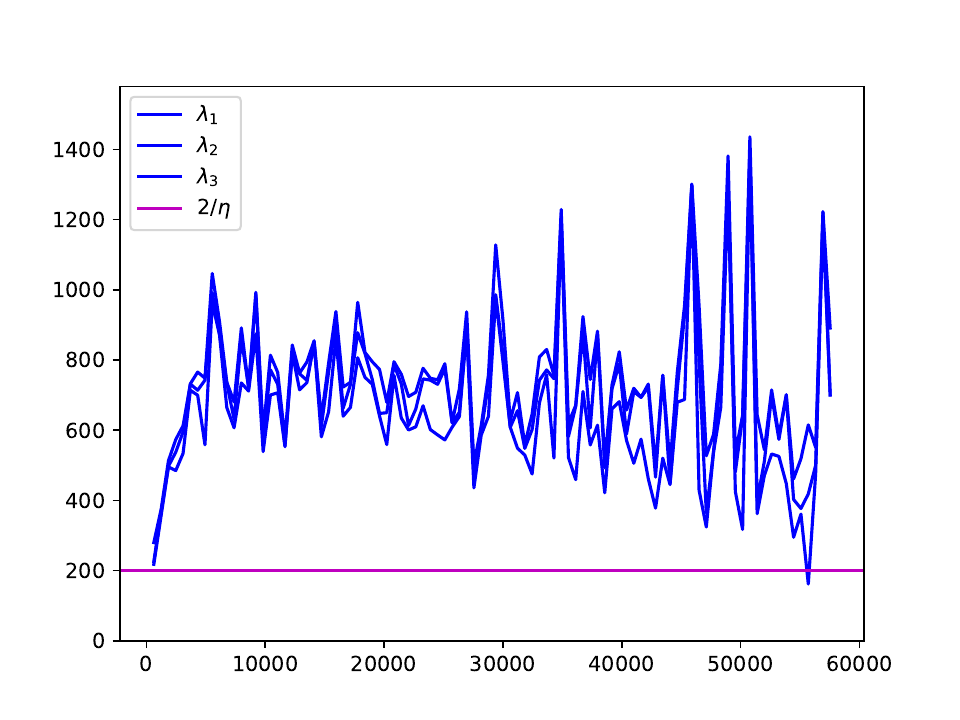}
        \caption{$\eta=0.01$}
    \end{subfigure}
    \begin{subfigure}{0.3\linewidth}
        \includegraphics[width=\linewidth]{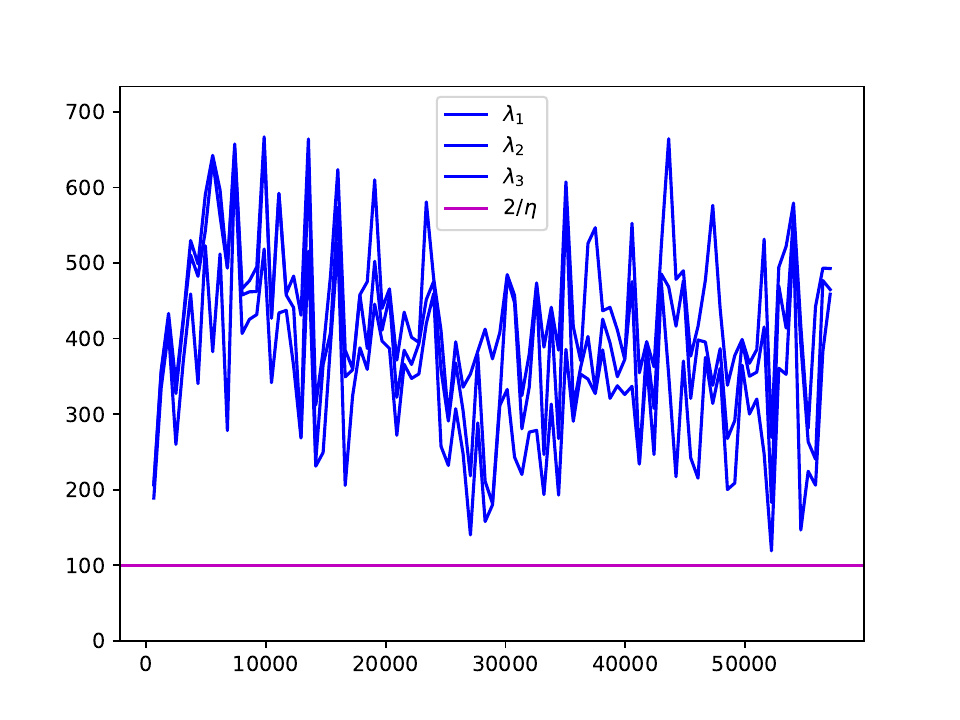}
        \caption{$\eta=0.02$}
    \end{subfigure}
    \begin{subfigure}{0.3\linewidth}
        \includegraphics[width=\linewidth]{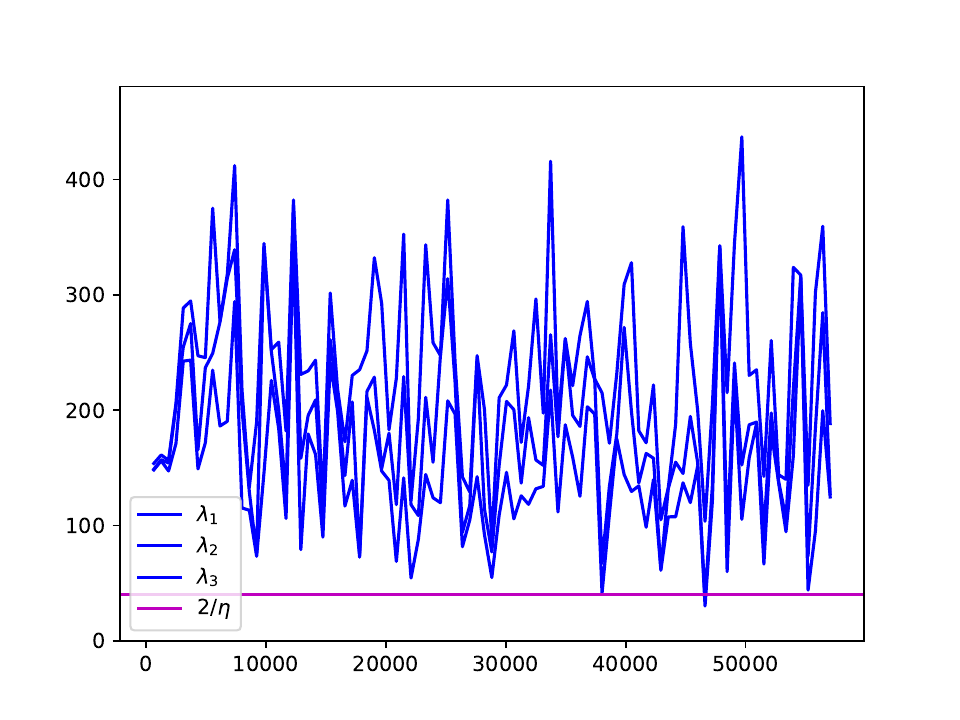}
        \caption{$\eta=0.05$}
    \end{subfigure}

    \begin{subfigure}{0.3\linewidth}
        \includegraphics[width=\linewidth]{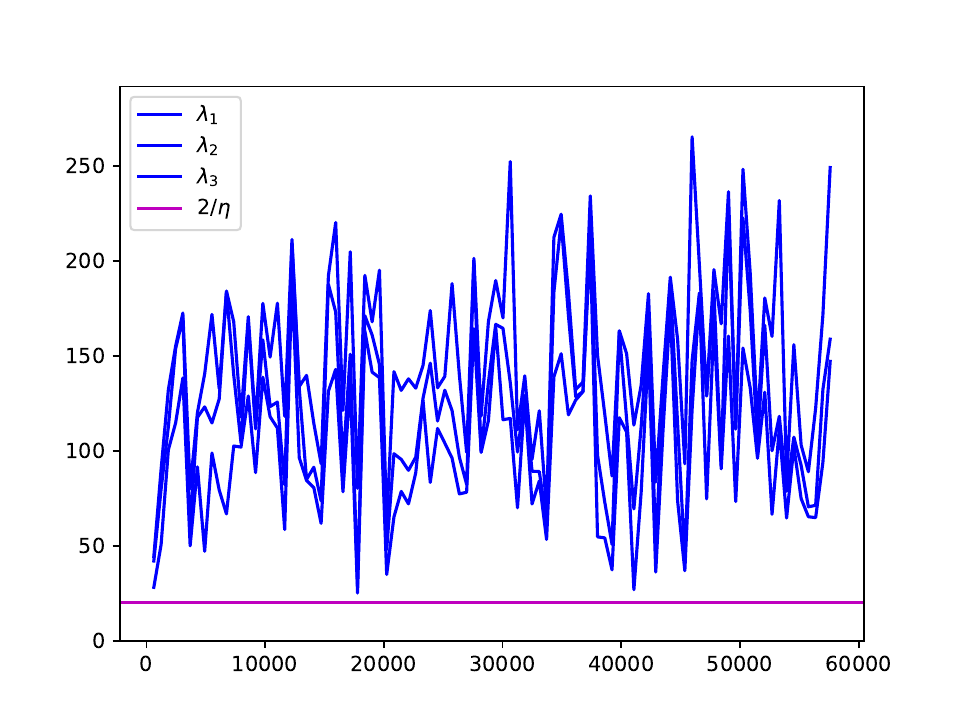}
        \caption{$\eta=0.1$}
    \end{subfigure}
    \begin{subfigure}{0.3\linewidth}
        \includegraphics[width=\linewidth]{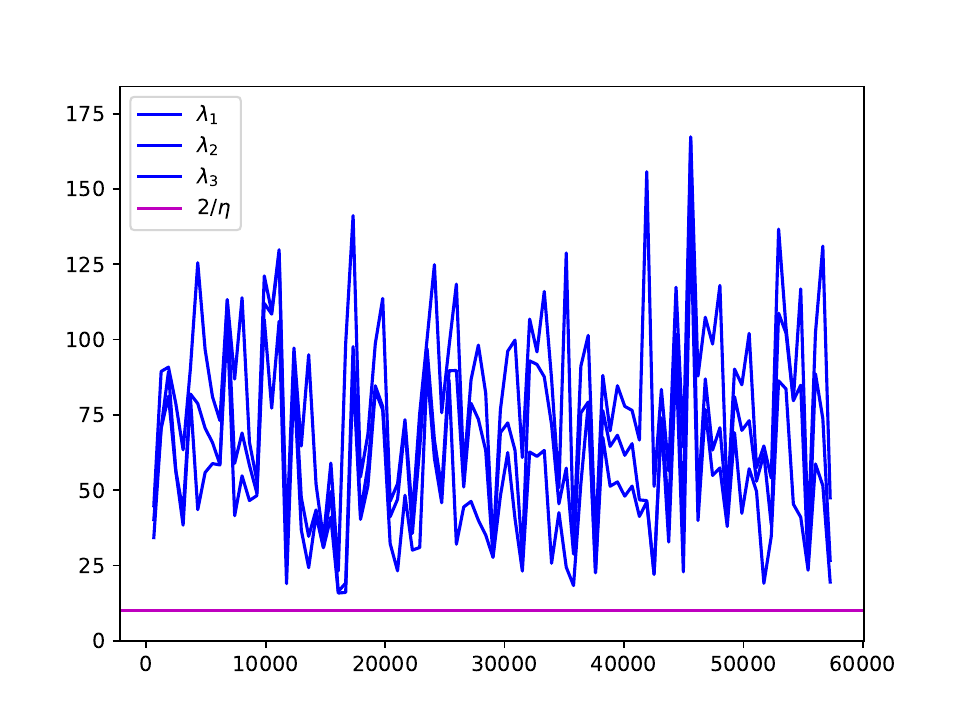}
        \caption{$\eta=0.2$}
    \end{subfigure}
    
    \caption{Magnitudes of the largest eigenvalues of the Hessian when a language model is trained with SGD.}
    \label{f:lm.rho=0.eigs}
\end{figure}

Next, we plot the same quantities when the
network is trained with SAM, with $\rho = 0.3$,
in Figure~\ref{f:lm.rho=0.3.eigs}.
\begin{figure}
    \centering
    \begin{subfigure}{0.3\linewidth}
        \includegraphics[width=\linewidth]{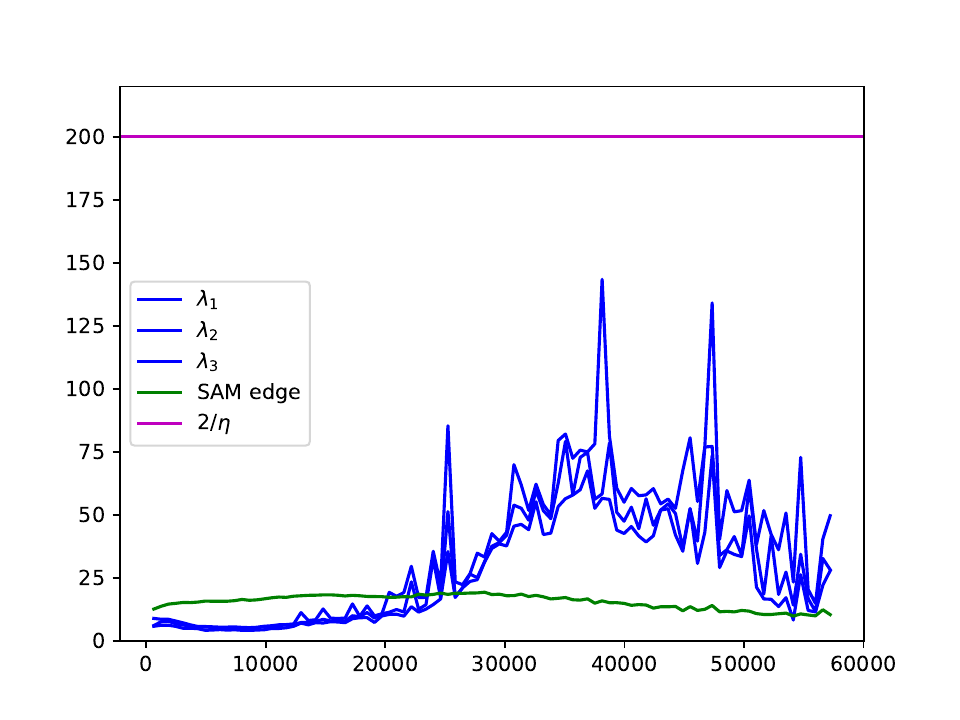}
        \caption{$\eta=0.01$}
    \end{subfigure}
    \begin{subfigure}{0.3\linewidth}
        \includegraphics[width=\linewidth]{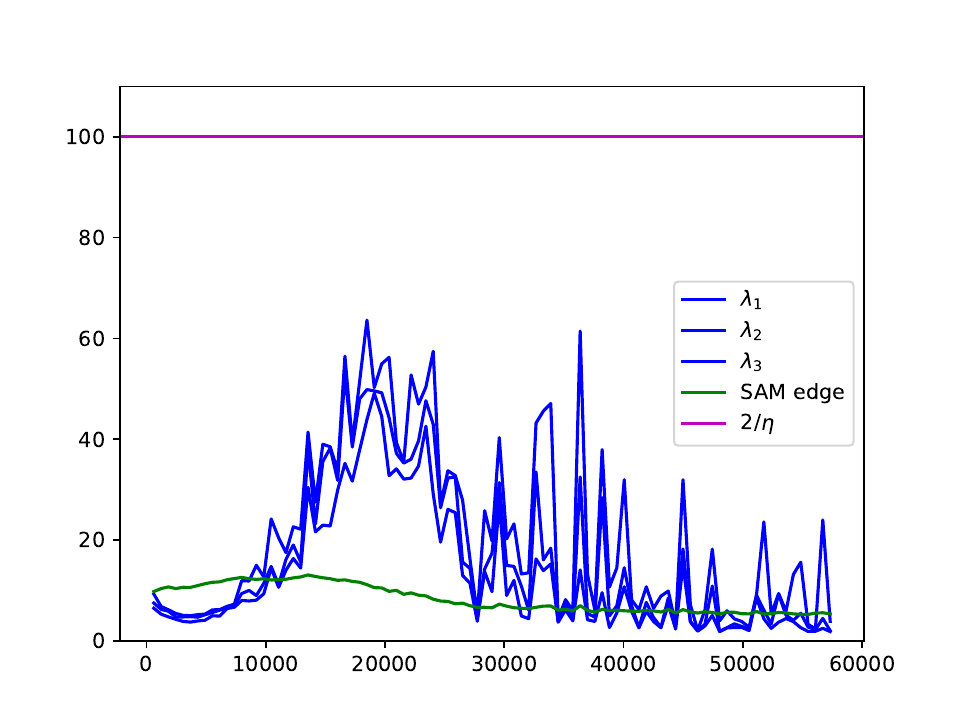}
        \caption{$\eta=0.02$}
    \end{subfigure}
    \begin{subfigure}{0.3\linewidth}
        \includegraphics[width=\linewidth]{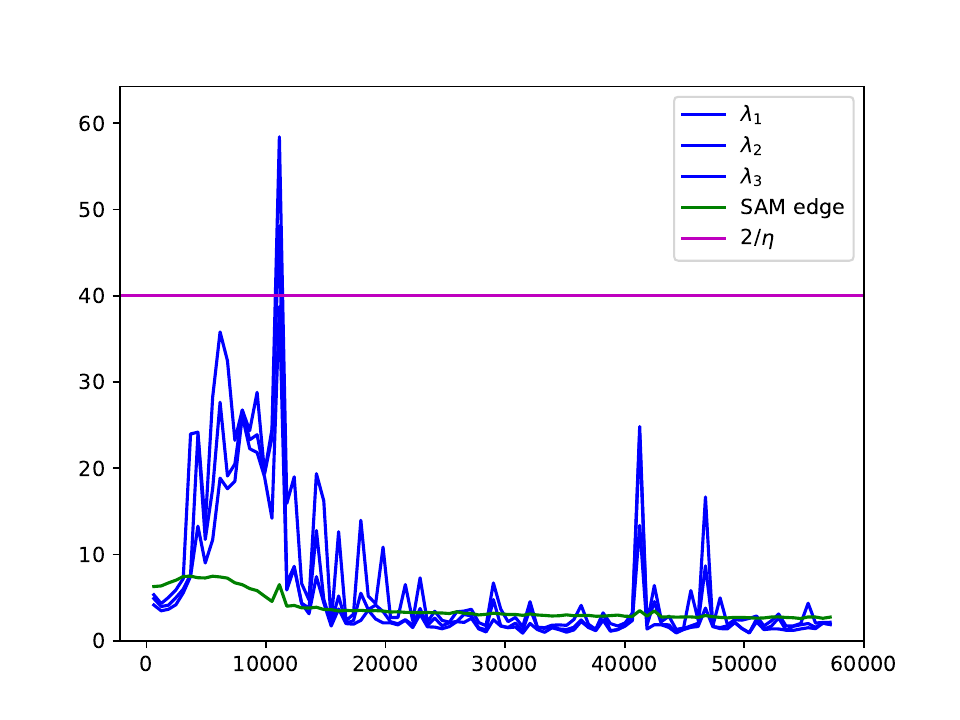}
        \caption{$\eta=0.05$}
    \end{subfigure}

    \begin{subfigure}{0.3\linewidth}
        \includegraphics[width=\linewidth]{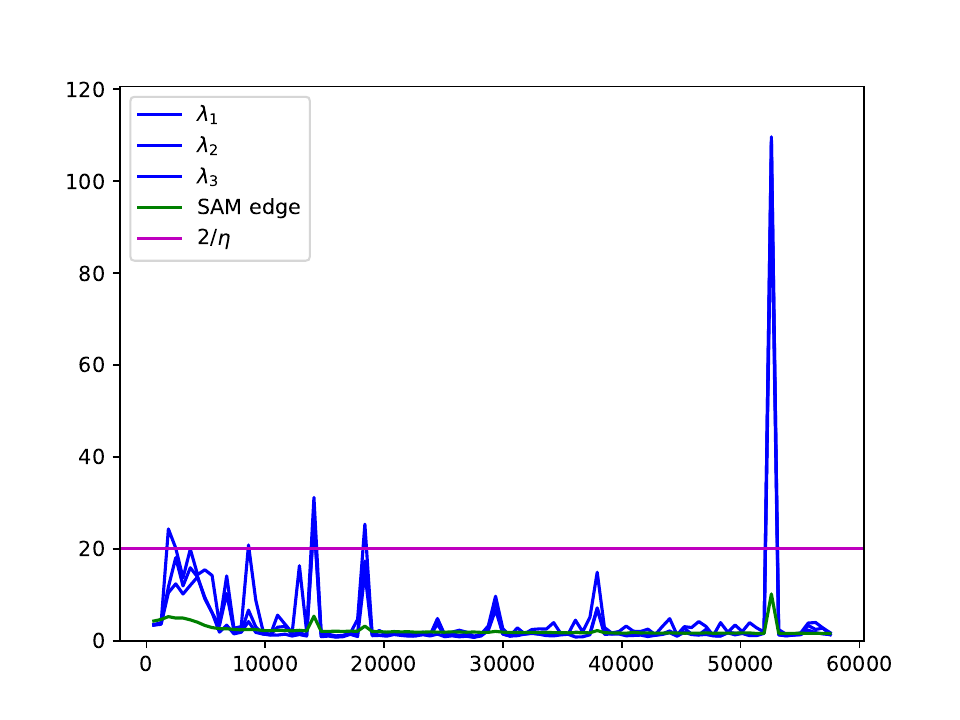}
        \caption{$\eta=0.1$}
    \end{subfigure}
    \begin{subfigure}{0.3\linewidth}
        \includegraphics[width=\linewidth]{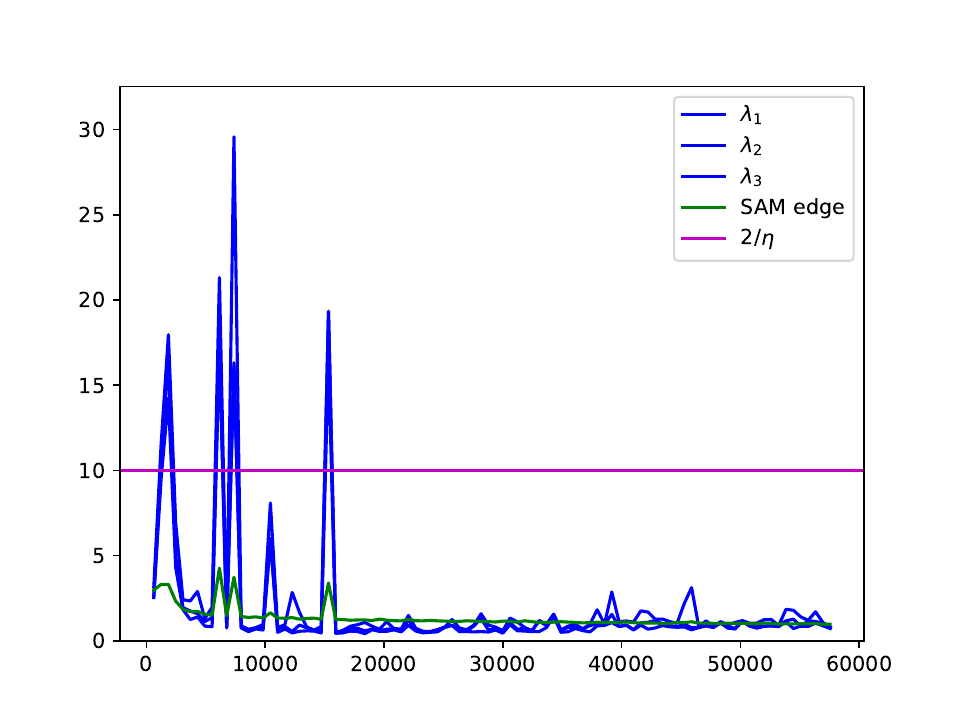}
        \caption{$\eta=0.2$}
    \end{subfigure}
    \begin{subfigure}{0.3\linewidth}
        \includegraphics[width=\linewidth]{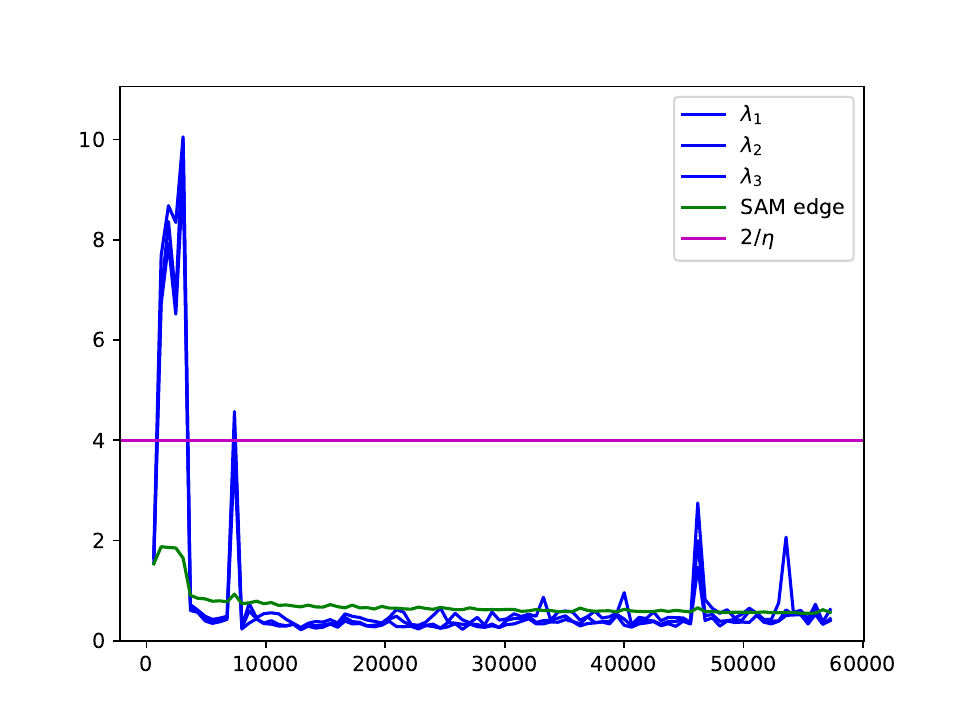}
        \caption{$\eta=0.5$}
    \end{subfigure}
    
    \caption{Magnitudes of the largest eigenvalues of the Hessian when a language model is trained with
    SAM, with $\rho = 0.3$.}
    \label{f:lm.rho=0.3.eigs}
\end{figure}
Here, the operator norm of the Hessian is significantly less
than when SGD is used, and we see evidence that training
in SAM operates at the edge of stability analyzed in Section~\ref{s:derivation}.
In Figure~\ref{f:lm.rho=0.3.eigs.zoom}, we zoom in on the lower part of
the curve, and plot the operator norm of the Hessian, to examine the relationship between this quantity and the SAM edge in more detail.
\begin{figure}
    \centering
    \begin{subfigure}{0.3\linewidth}
        \includegraphics[width=\linewidth]{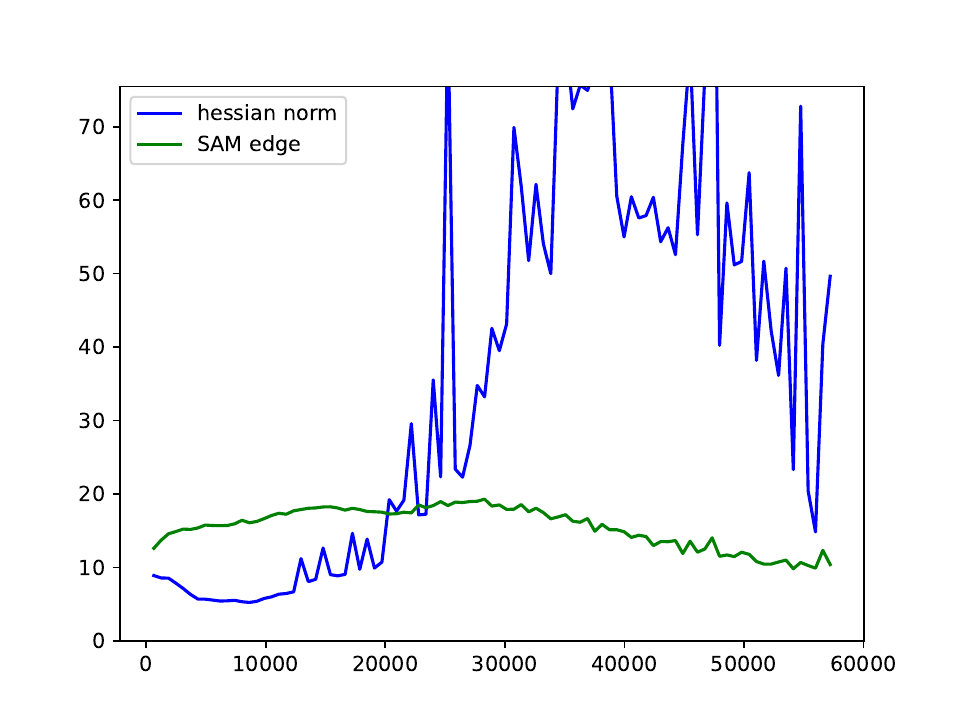}
        \caption{$\eta=0.01$}
    \end{subfigure}
    \begin{subfigure}{0.3\linewidth}
        \includegraphics[width=\linewidth]{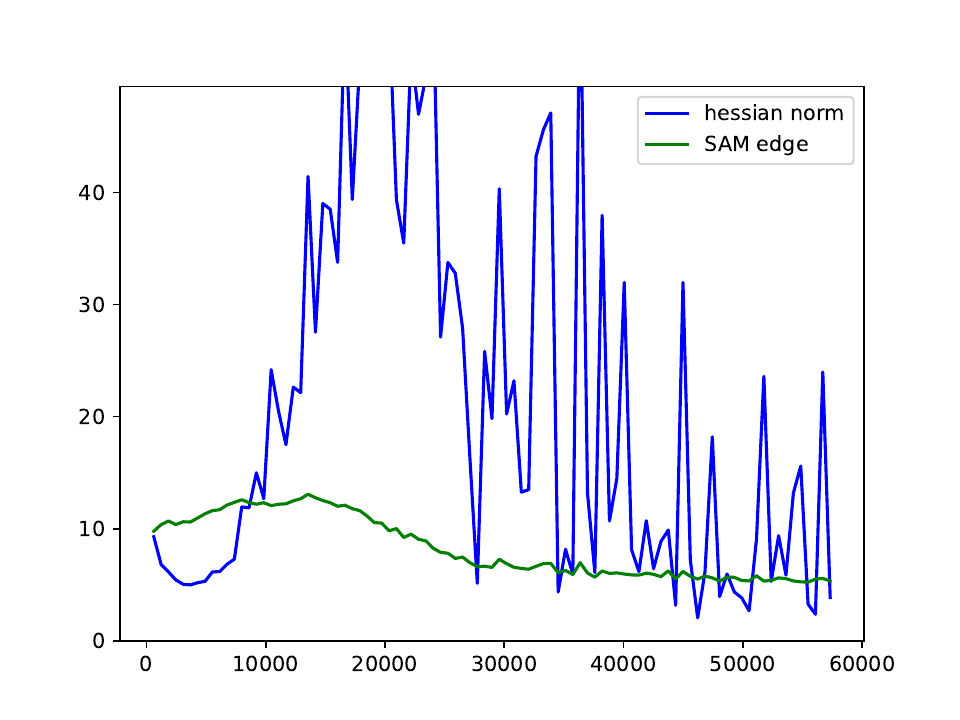}
        \caption{$\eta=0.02$}
    \end{subfigure}
    \begin{subfigure}{0.3\linewidth}
        \includegraphics[width=\linewidth]{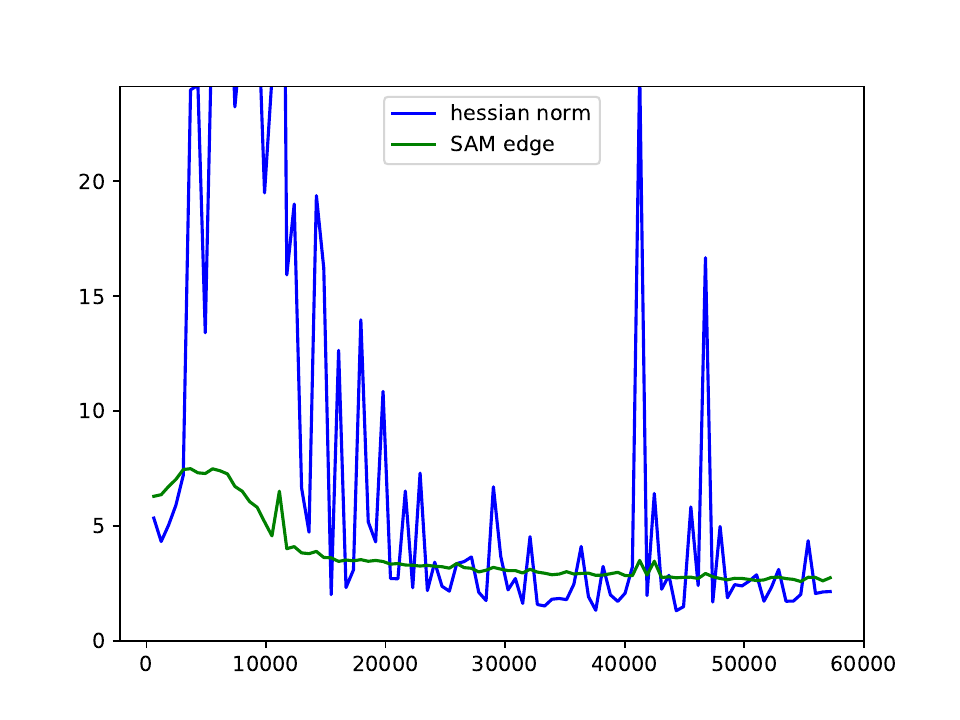}
        \caption{$\eta=0.05$}
    \end{subfigure}

    \begin{subfigure}{0.3\linewidth}
        \includegraphics[width=\linewidth]{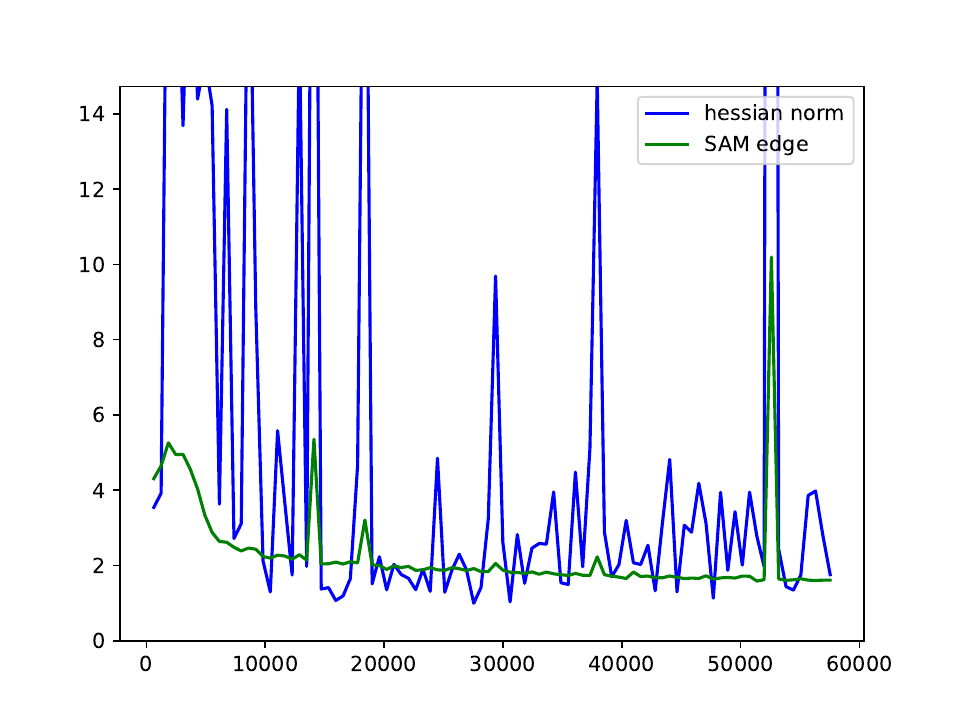}
        \caption{$\eta=0.1$}
    \end{subfigure}
    \begin{subfigure}{0.3\linewidth}
        \includegraphics[width=\linewidth]{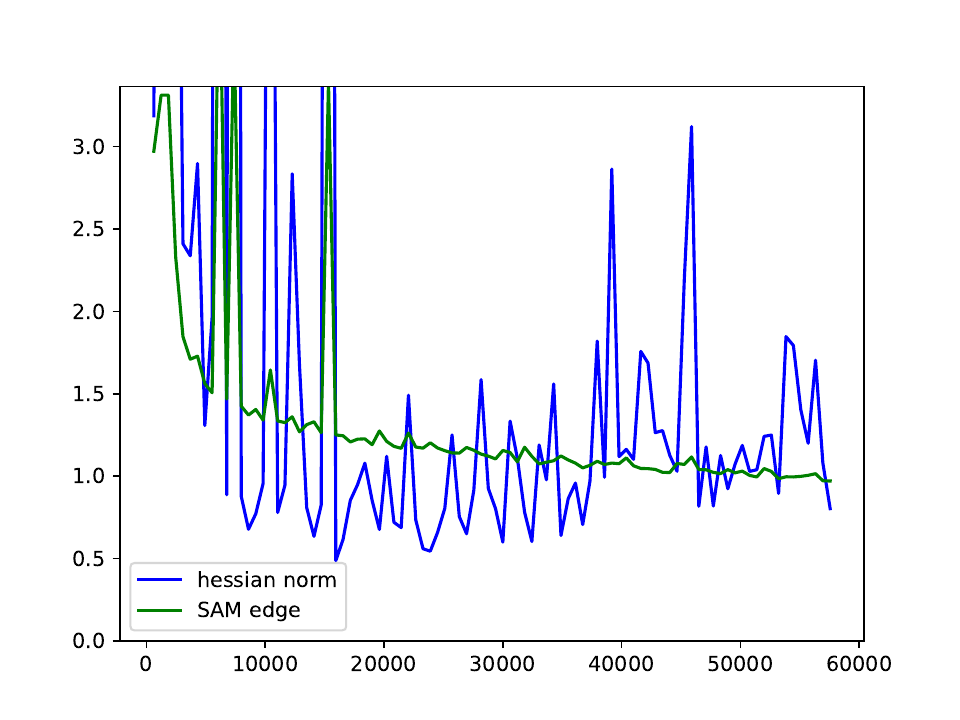}
        \caption{$\eta=0.2$}
    \end{subfigure}
    \begin{subfigure}{0.3\linewidth}
        \includegraphics[width=\linewidth]{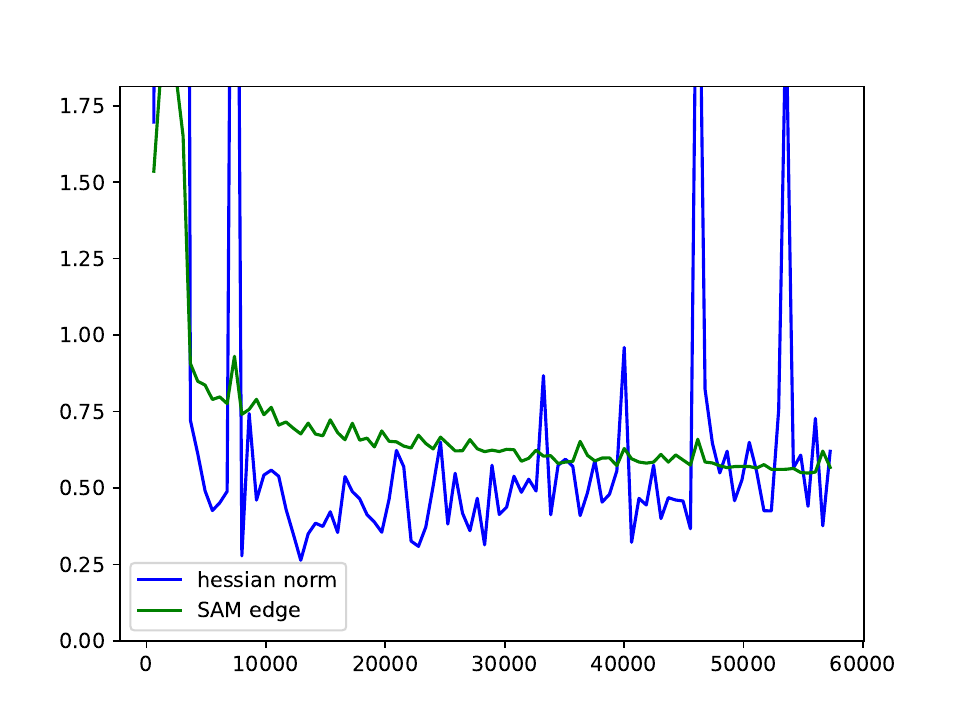}
        \caption{$\eta=0.5$}
    \end{subfigure}
    
    \caption{Magnitudes of the largest eigenvalues of the Hessian when a language model is trained with
    SAM, with $\rho = 0.3$.}
    \label{f:lm.rho=0.3.eigs.zoom}
\end{figure}

Figure~\ref{f:lm.loss} contains plots of
the training loss, once again estimated per-minibatch.
\begin{figure}
    \centering
        \begin{subfigure}{0.3\linewidth}
        \includegraphics[width=\linewidth]{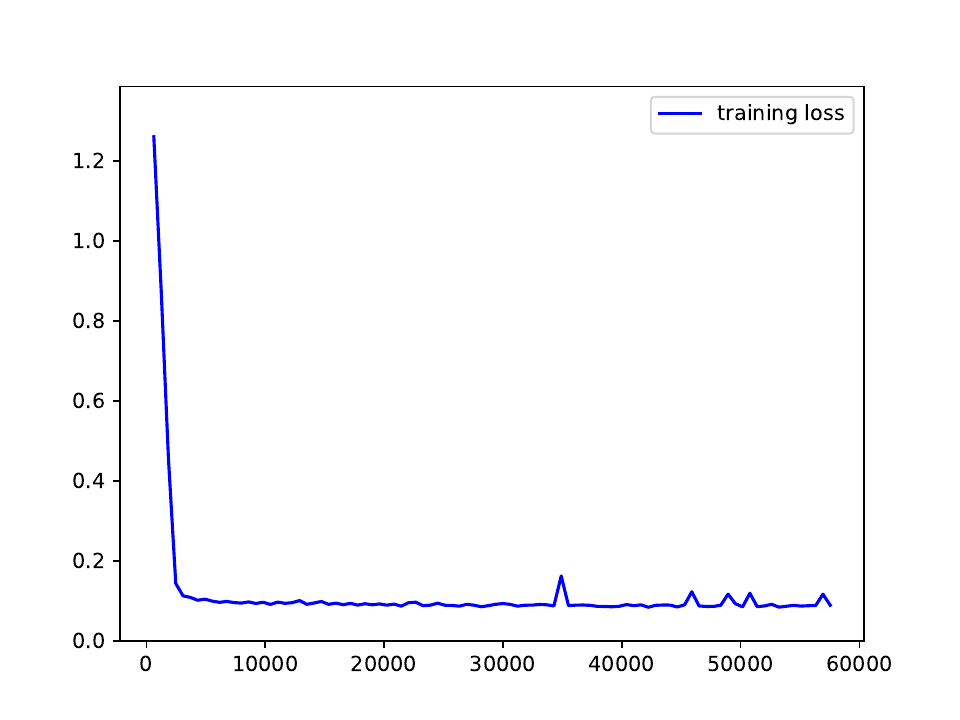}
        \caption{$\rho = 0, \eta=0.01$}
    \end{subfigure}
    \begin{subfigure}{0.3\linewidth}
        \includegraphics[width=\linewidth]{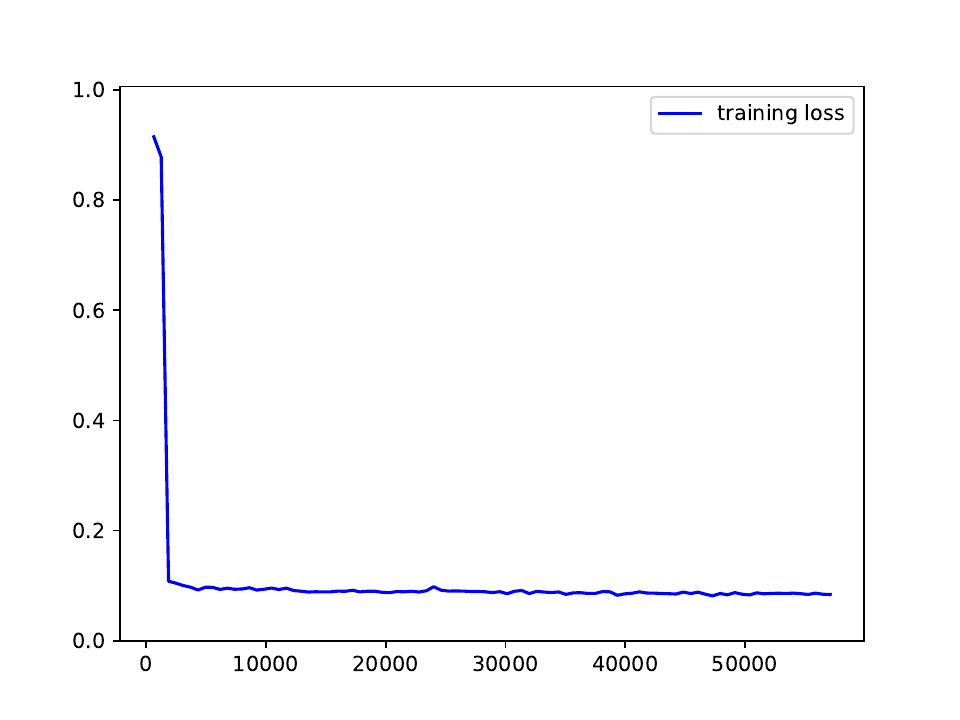}
        \caption{$\rho = 0, \eta=0.02$}
    \end{subfigure}
    \begin{subfigure}{0.3\linewidth}
        \includegraphics[width=\linewidth]{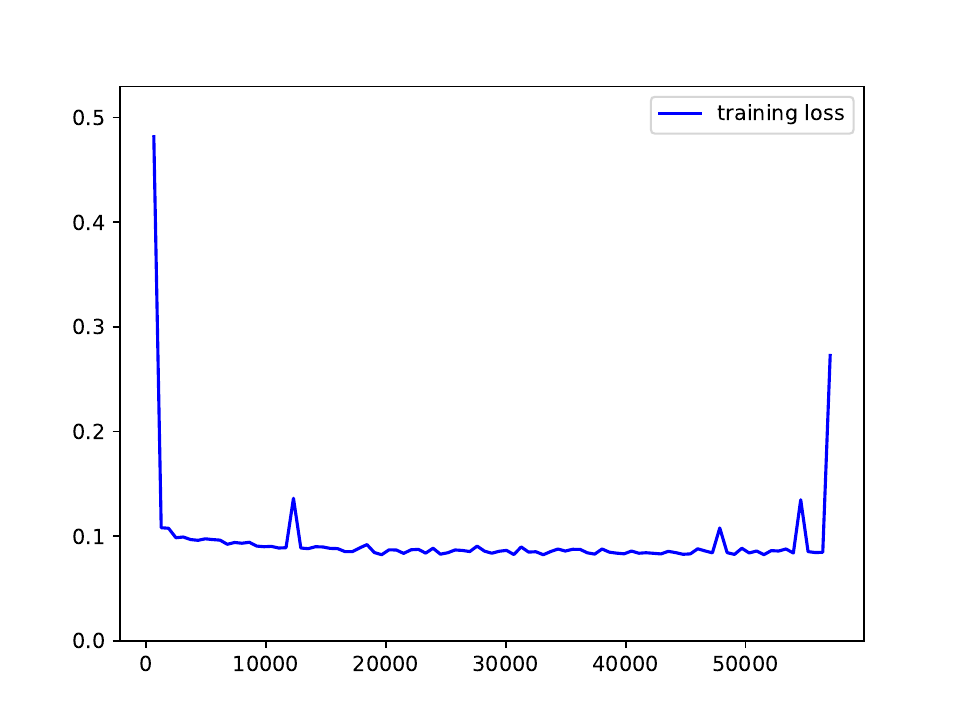}
        \caption{$\rho = 0, \eta=0.05$}
    \end{subfigure}
    \begin{subfigure}{0.3\linewidth}
        \includegraphics[width=\linewidth]{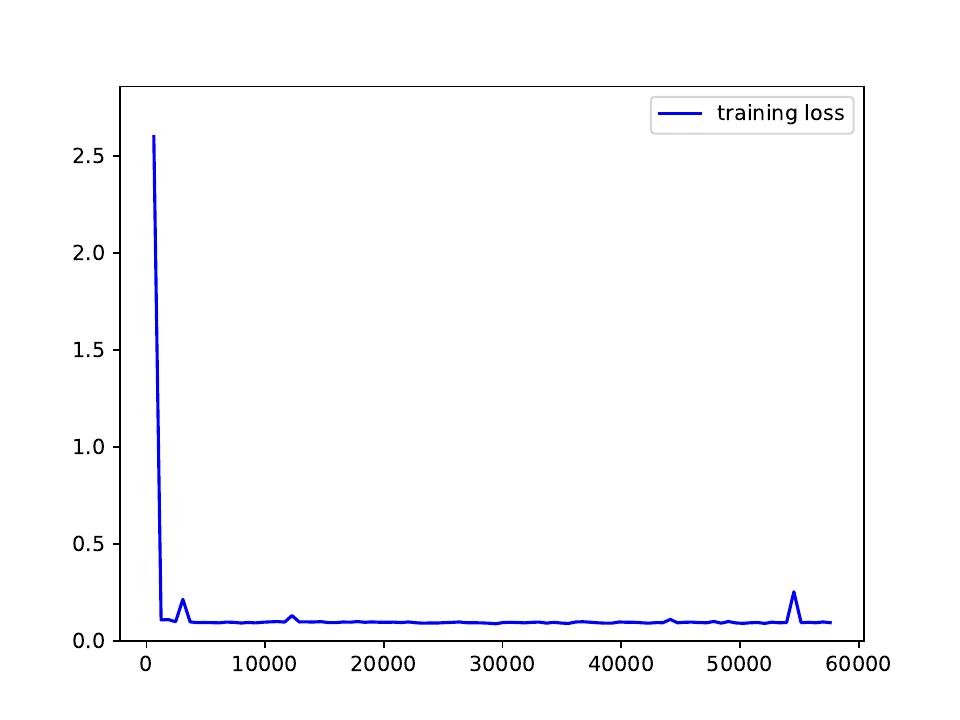}
        \caption{$\rho = 0, \eta=0.1$}
    \end{subfigure}
    \begin{subfigure}{0.3\linewidth}
        \includegraphics[width=\linewidth]{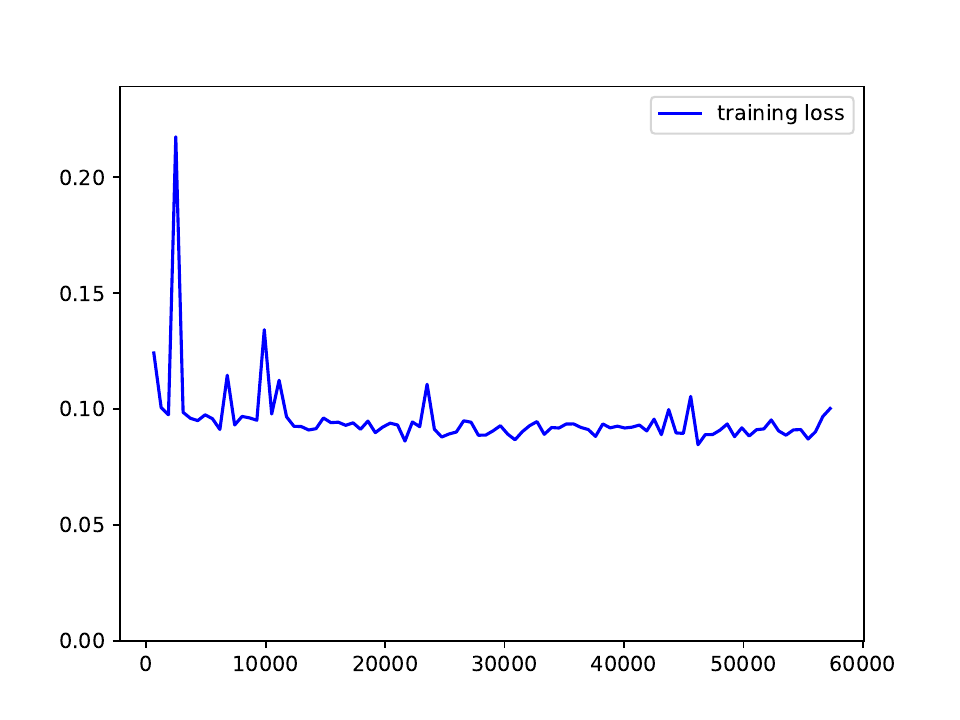}
        \caption{$\rho = 0, \eta=0.2$}
    \end{subfigure}
    \begin{subfigure}{0.3\linewidth}
        \includegraphics[width=\linewidth]{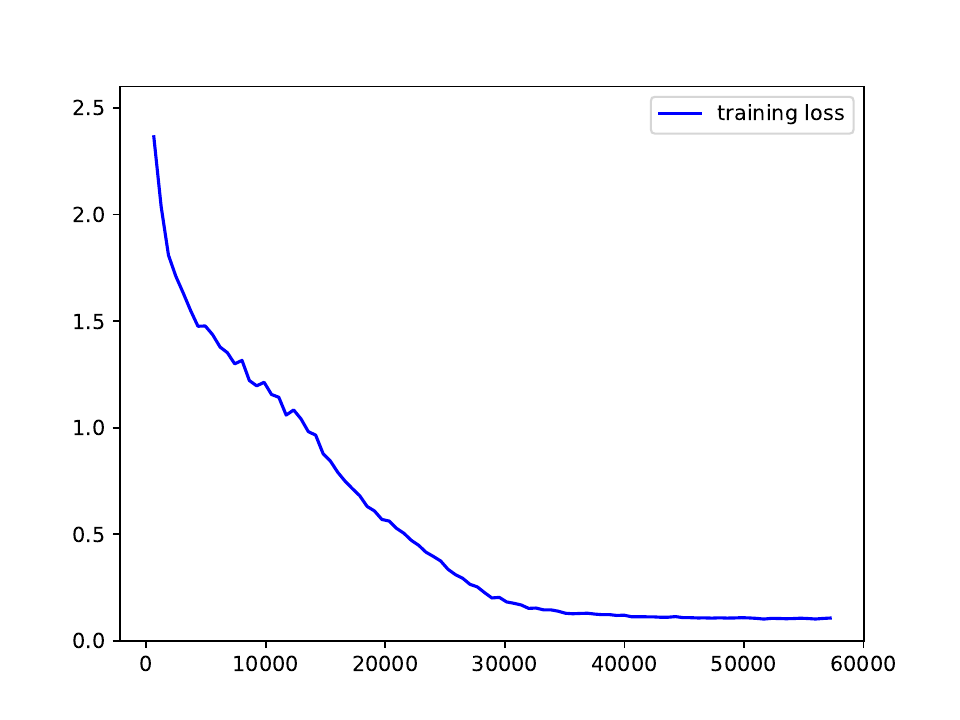}
        \caption{$\rho = 0.3, \eta=0.01$}
    \end{subfigure}
    \begin{subfigure}{0.3\linewidth}
        \includegraphics[width=\linewidth]{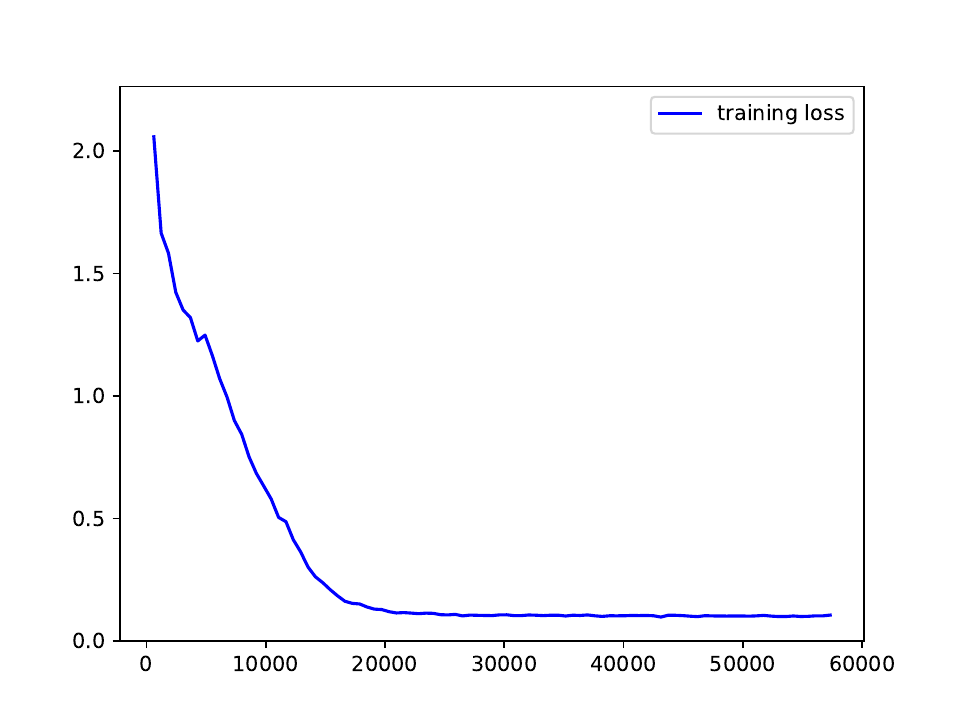}
        \caption{$\rho = 0.3, \eta=0.02$}
    \end{subfigure}
    \begin{subfigure}{0.3\linewidth}
        \includegraphics[width=\linewidth]{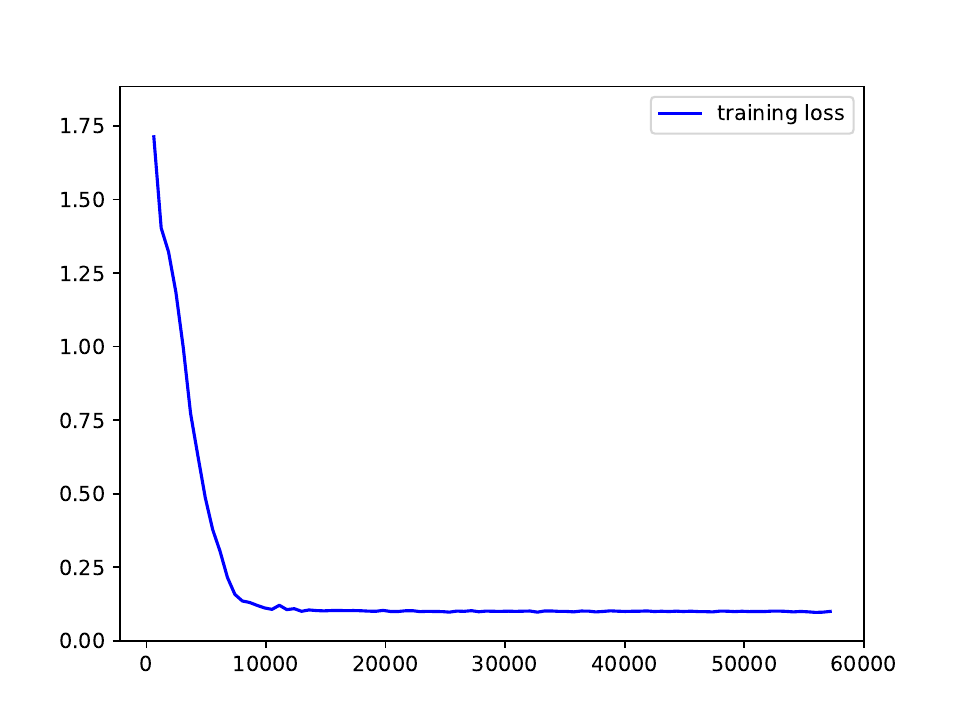}
        \caption{$\rho = 0.3, \eta=0.05$}
    \end{subfigure}
    \begin{subfigure}{0.3\linewidth}
        \includegraphics[width=\linewidth]{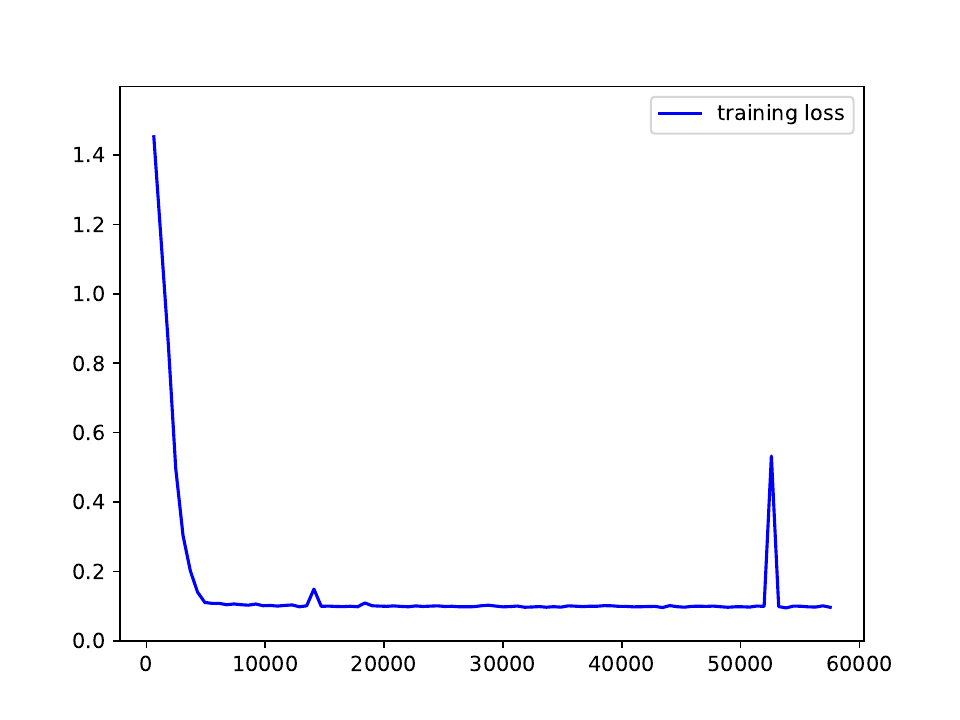}
        \caption{$\rho = 0.3, \eta=0.1$}
    \end{subfigure}
    \begin{subfigure}{0.3\linewidth}
        \includegraphics[width=\linewidth]{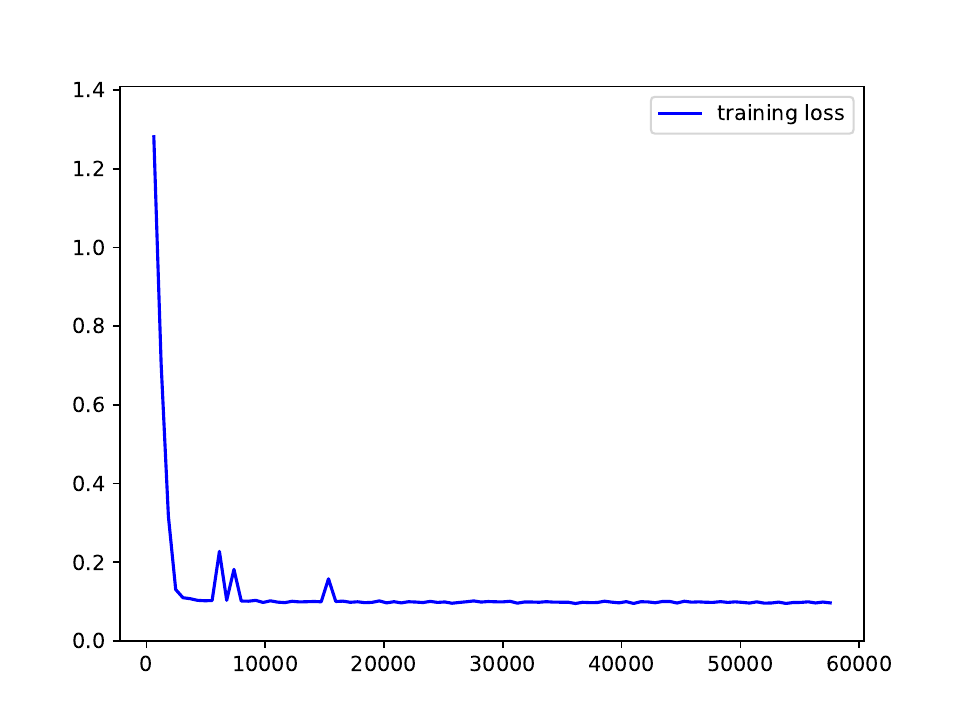}
        \caption{$\rho = 0.3, \eta=0.2$}
    \end{subfigure}
    \begin{subfigure}{0.3\linewidth}
        \includegraphics[width=\linewidth]{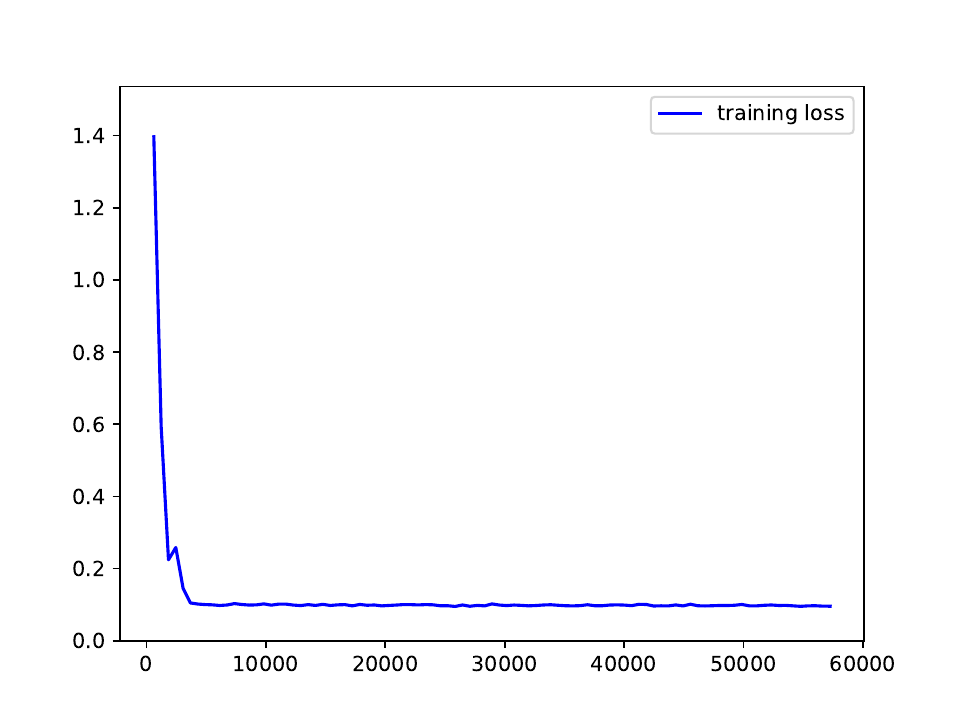}
        \caption{$\rho = 0.3, \eta=0.5$}
    \end{subfigure}
    \caption{Training loss in the language modeling experiments.}
    \label{f:lm.loss}
\end{figure}
We included these mainly to motivate the combinations
of hyperparameters where we examined other aspects of
the dynamics of SAM.  As expected,
while SAM does take longer to achieve
a certain loss, it ultimately achieves training
error similar to SGD, but with less sharpness.

Figure~\ref{f:lm.alignment} contains plots of
the alignment, once again estimated per-minibatch.
\begin{figure}
    \centering
    \begin{subfigure}{0.3\linewidth}
        \includegraphics[width=\linewidth]{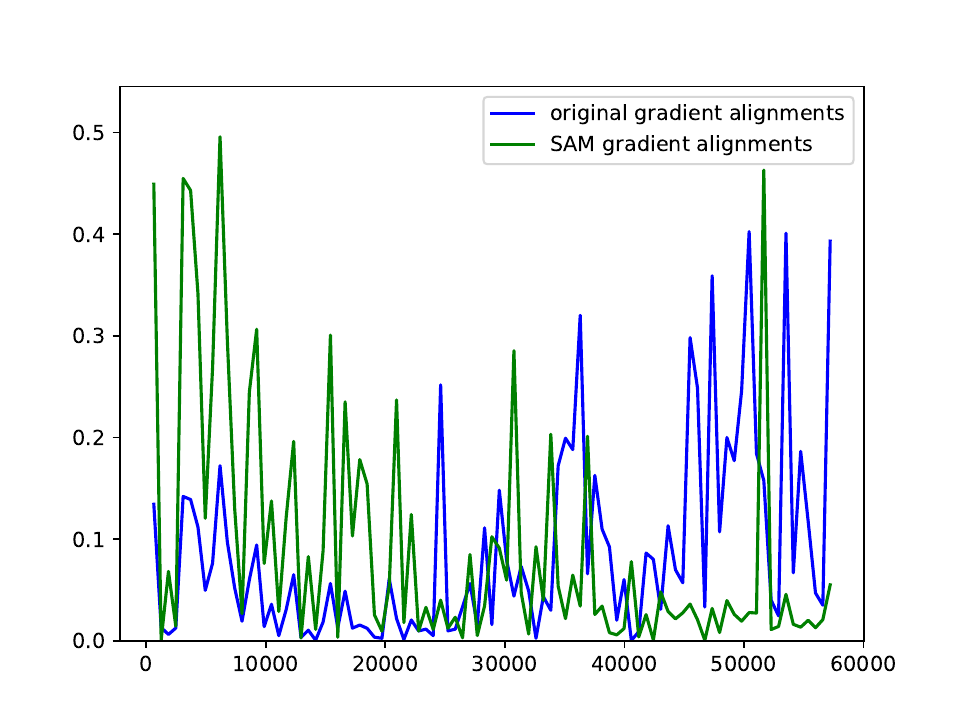}
        \caption{$\rho = 0.3, \eta=0.01$}
    \end{subfigure}
    \begin{subfigure}{0.3\linewidth}
        \includegraphics[width=\linewidth]{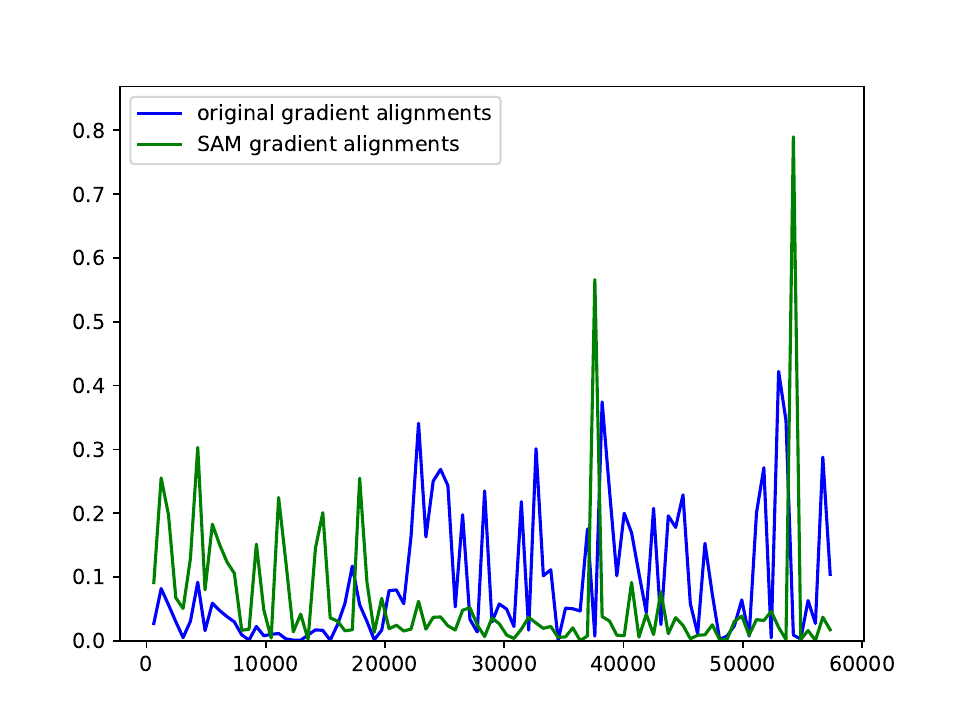}
        \caption{$\rho = 0.3, \eta=0.02$}
    \end{subfigure}
    \begin{subfigure}{0.3\linewidth}
        \includegraphics[width=\linewidth]{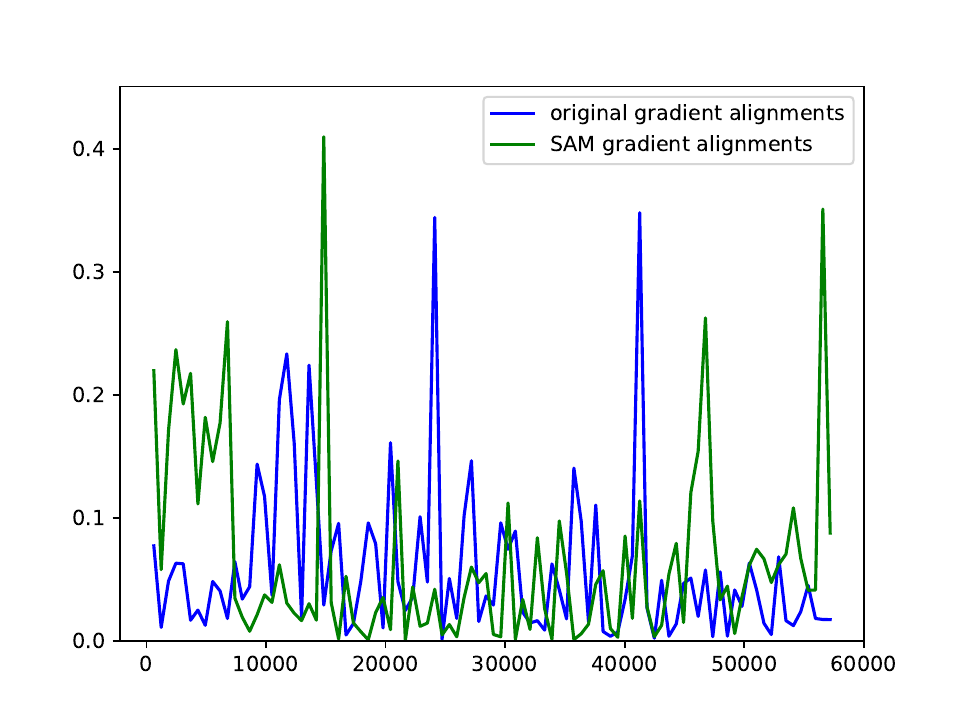}
        \caption{$\rho = 0.3, \eta=0.05$}
    \end{subfigure}
    \begin{subfigure}{0.3\linewidth}
        \includegraphics[width=\linewidth]{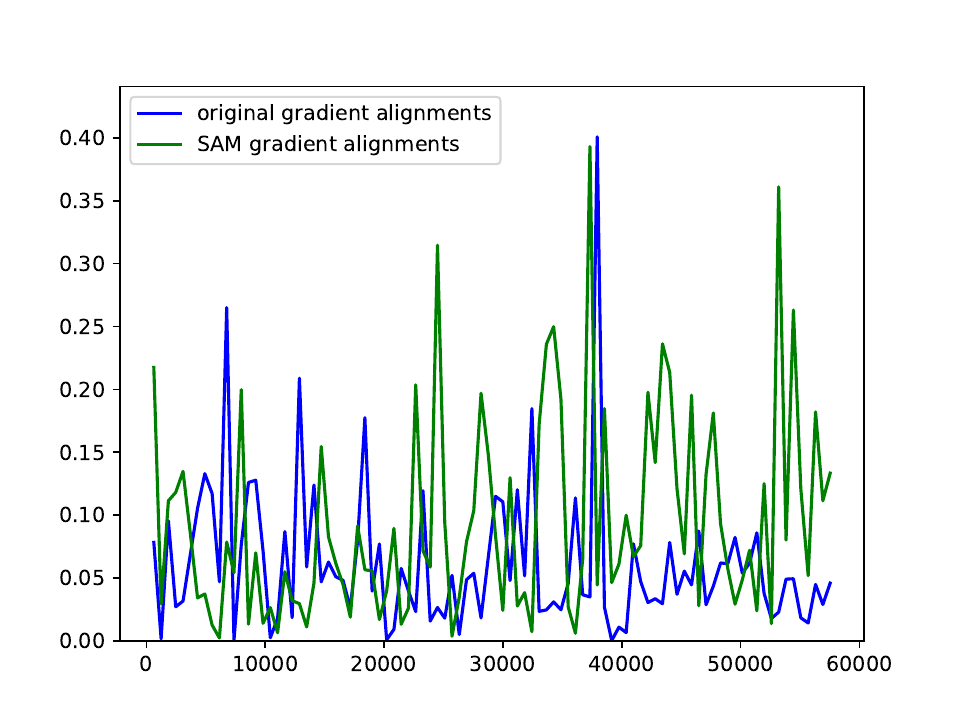}
        \caption{$\rho = 0.3, \eta=0.1$}
    \end{subfigure}
    \begin{subfigure}{0.3\linewidth}
        \includegraphics[width=\linewidth]{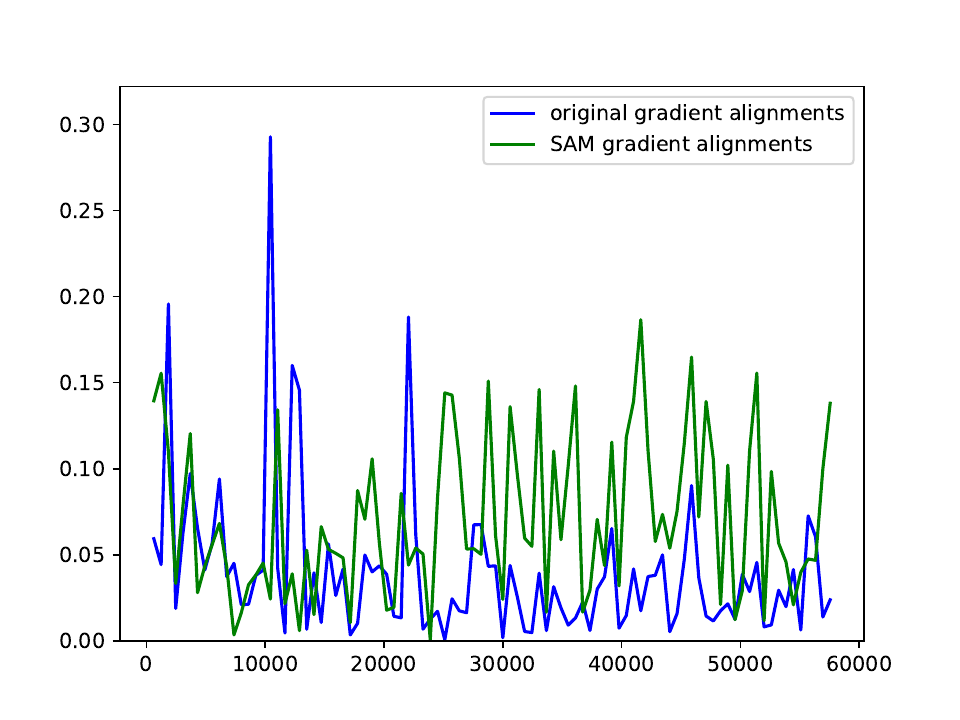}
        \caption{$\rho = 0.3, \eta=0.2$}
    \end{subfigure}
    \begin{subfigure}{0.3\linewidth}
        \includegraphics[width=\linewidth]{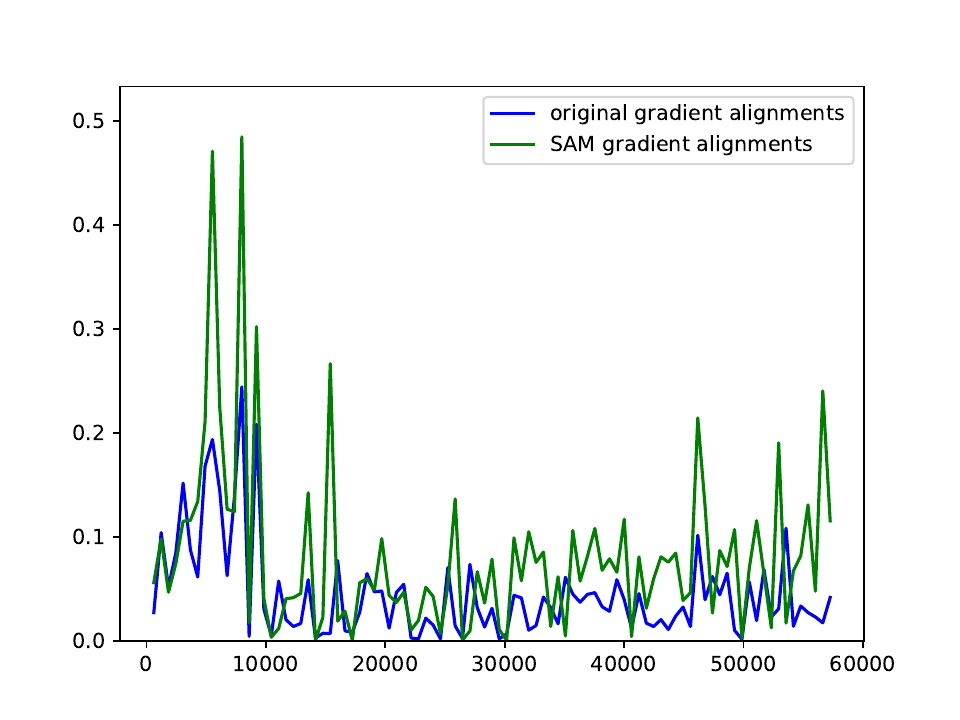}
        \caption{$\rho = 0.3, \eta=0.5$}
    \end{subfigure}
    \caption{Alignments between gradients and the principal direction of the Hessian in the language modeling experiments.}
    \label{f:lm.alignment}
\end{figure}
For the large learning rates, late in training, despite the sampling
noise arising from the use of minibatches, we see a systematic
tendency for the SAM gradients to align more closely with the
principal eigenvector of the Hessian than the gradients at the initial
solution.  However, for the smallest learning rates, the {\em opposite}
holds.  

\section{Related work}
\label{s:related}

In this section, we describe some previously mentioned related work in more detail,
and also go over some additional papers.

\citet{bartlett2022dynamics} analyzed the dynamics of SAM applied to
a convex quadratic objective, and showed that it converges to oscillating
in the direction of the principal eigenvector.  Then they analyzed one
step of SAM in more generality, starting at a solution near a local
minimum, analogous to one of the steady-state solutions in the convex
quadratic case.  They showed that the update from this point can
be decomposed into three terms, a term that corresponds to the update in
the convex quadratic case (which moves to the other solution in the oscillation),
a term in the descent direction of the operator norm of the Hessian, and a
third term, which, for small $\eta$ and $\rho$, is of lower order.  
The edge-of-stability point identified here is not a consequence of that analysis.

Among the varied results of \citet{wen2023how} is a theorem that may be
paraphrased by saying that, for a smooth enough objective functions, 
in an overparameterized regime where there is
a manifold of minimizers, once SAM's iterates are close to this manifold, its updates
track the updates that would be obtained by performing gradient flow to minimize
the operator norm of the Hessian among minimizers of the loss.  
Their main results use the assumptions that
$\eta \log(1/\rho)$ and $\rho/\eta$ are sufficiently small.
As was seen by \citet{cohen2021gradient} and also here, 
the edge-of-stability phenomenon dissipates as $\eta$ gets small.

\citet{andriushchenko2023sharpness} demonstrated empirically that networks
trained by SAM tend to have features with lower rank, and illustrated how this
can arise using a theoretical analysis of a two-layer network.

\citet{cohen2022adaptive} demonstrated that some adaptive gradient methods, such
as Adam, operate at the edge of stability.

A number of authors have provided insight by analyzing the dynamics of gradient descent under clean and simple conditions under which the edge of stability
arises 
\citep[see][]{zhu2022understanding,DBLP:conf/icml/AgarwalaPP23,ahn2022learning,chen2022gradient,even2023s}.
Properties of the loss landscape that are compatible with
edge of stability training have also been described (and evaluated
empirically)   
\citep{ma2022beyond,ahn2022understanding}.
%
\citet{arora2022understanding} established conditions under which an algorithm
like GD, but that normalizes the gradients so that they have
unit length, operates at the edge of stability, and also analyzed
an algorithm that takes gradients with respect to the square root
of the loss.

Some authors have studied an algorithm like SAM, but, instead of updating
using the gradient from the neighbor of the current iterate that is a
constant distance $\rho$ uphill, instead uses a
gradient from
neighbor whose distance from the current iterate scales with
the norm of the gradient at the iterate \citep{andriushchenko2022towards,agarwala2023sam},
what has been called ``unnormalized SAM''. 
\citet{dai2023crucial} made a case that the SAM's normalization is
crucial, motivating research into the original algorithm.

\section{Conclusion}
\label{s:conclusion}

We have computed the critical value of operator norm of the Hessian corresponding to the edge of stability for SAM.   This SAM-edge is a decreasing function
of the norm of the gradient, so it tends to decrease as training 
progresses.
For three deep learning training tasks, we have seen that the
operator norm of the Hessian closely tracks this edge
of stability, despite the noise introduced by estimating using
minibatches in the \verb|tiny_shakespeare| task.

SAM interacts strongly with the edge-of-stability phenomenon to drive down the
operator norm of the Hessian, while also driving down the training error.  
Insight into how and why this happens could be promoted by
identifying conditions under which SAM provably operates at its
edge of stability, analogous to the results obtained for
GD mentioned in Section~\ref{s:related}.
The analyses of
\citet{bartlett2022dynamics} and \citet{wen2023how} both required
$\eta$ and $\rho$ to be small, and analyzed the effect of the dynamics
on the operator norm of the Hessian late in training, whereas we empirically see
a strong effect even early in training.  One especially interesting question is
how the training error is reduced so rapidly despite the overshooting associated with
edge-of-stability training.

The experiments with language models showed that the edge-of-stability phenomenon
can also be seen, to a limited extent, when training with SGD.  A more thorough
understanding of SAM and the edge of stability when training with SGD is another
interesting and important subject for further research.  
(\citet{wen2023how} analyzed a variant of SAM that works using SGD one example
at a time, and pointed out strong qualitative differences between the algorithm
that works with batch gradients and this extreme version of SGD, suggesting that
interesting and rich structure might be found in the behavior of SAM with minibatches
of intermediate size.)

In our experiments, there was a general tendency
for the gradients used by SAM to be
more aligned with the principal direction
of the Hessian than gradients evaluated
at the iterates.  It is not clear why
this is the case, and under what conditions
it happens.  The theoretical analysis by \citet{bartlett2022dynamics}
depended critically on the assumption that the update gradient was aligned
with the principal eigenvector of the Hessian, which raises the possibility that
the fact that the gradients used by SAM are aligned more closely with the
principal direction of the Hessian is key to its success.  However, it is not
clear under what conditions, and why, this improved alignment is seen, and when
it is helpful.
There also was an intriguing exception
when language models were trained with SGD
using small step sizes that it would be interesting
to further explore.



\section*{Acknowledgements}

We thank Naman Agarwal and Hossein Mobahi for valuable conversations, and Naman Agarwal for his comments on an earlier version of this paper.
PB gratefully acknowledges the support of the NSF through grants
DMS-2023505 and DMS-2031883 and of Simons Foundation award \#814639.

\appendix


\vskip 0.2in
\bibliographystyle{plainnat}
\bibliography{bib}

\end{document}